\documentclass[lettersize,journal]{IEEEtran}

\usepackage{hyperref}

\usepackage[latin1]{inputenc}
\usepackage{amsmath}
\usepackage{amsfonts}
\usepackage{amssymb}
\usepackage{amstext}
\usepackage{amsthm}
\usepackage{caption}
\usepackage{subcaption}
\usepackage{graphicx}
\usepackage{algorithmic}
\usepackage{algorithm}
\usepackage{url}
\usepackage{color}
\usepackage{xcolor}
\usepackage{bm}
\usepackage{dsfont}
\usepackage{soul}
\usepackage{multirow}
\usepackage{booktabs}
\usepackage{makecell}
\usepackage{comment}

\newtheorem{theorem}{Theorem}

\newtheorem{proposition}{Proposition}

\usepackage{color}
\usepackage{xcolor,colortbl}

\colorlet{mygreen}{green!60!gray}

\newcommand{\BT}[1]{\textcolor{black}{{#1}}}
\newcommand{\CH}[1]{\textcolor{black}{{#1}}}

\newif\ifcomments

\commentstrue 

\ifcomments

\newcommand{\BTcomm}[1]{\textcolor{mygreen}{Benedetta: {#1}}}
\newcommand{\TODO}[1]{\textcolor{red}{TODO: {#1}}}
\newcommand{\MB}[1]{\textcolor{red}{#1}}

\else 

\newcommand{\BTcomm}[1]{}
\newcommand{\TODO}[1]{}
\newcommand{\MB}[1]{}

\fi

\newcommand\ChangeRT[1]{\noalign{\hrule height #1}}

\begin{document}

\title{JMA: a General Algorithm to Craft Nearly Optimal Targeted Adversarial Examples}
%
%
%

\author{Benedetta Tondi,~\IEEEmembership{Senior Member,~IEEE}, Wei Guo, Niccol\`o Pancino,
 Mauro Barni,~\IEEEmembership{Fellow,~IEEE}
\thanks{B. Tondi, N.Pancino and M. Barni are from the Department of Information Engineering and Mathematics, University of Siena, 53100 Siena, Italy. W. Guo is from the Department of Electrical and Electronic Engineering, University of Cagliari, 09123 Cagliari, Italy.  Corresponding author: W. Guo (email: wei.guo.cn@outlook.com)}%
\thanks{This work has been partially supported by SERICS (PE00000014) under the MUR National Recovery and Resilience Plan funded by the European Union - NextGenerationEU
%
%
}
}

\markboth{Journal of \LaTeX\ Class Files,~Vol.~14, No.~8, August~2021}%
{Shell \MakeLowercase{\textit{et al.}}: A Sample Article Using IEEEtran.cls for IEEE Journals}


\maketitle

\vspace{-0.2cm}

\begin{abstract}
Most of the approaches proposed so far to craft targeted adversarial examples against Deep Learning classifiers are highly suboptimal and typically rely on increasing the likelihood of the target class, thus implicitly focusing on one-hot encoding settings. In this paper, a more general, theoretically sound, targeted attack is proposed, which resorts to the minimization of a Jacobian-induced  Mahalanobis distance term, taking into account the effort (in the input space) required to move the latent space representation of the input sample in a given direction. The minimization is solved by exploiting the Wolfe duality theorem, reducing the problem to the solution of a Non-Negative Least Square (NNLS) problem. The proposed algorithm (referred to as JMA) provides an optimal solution to a linearised version of the adversarial example problem originally introduced by Szegedy et al. The results of the experiments confirm the generality of the proposed attack which is proven to be effective under a wide variety of output encoding schemes. Noticeably, JMA is also effective in a multi-label classification scenario, being capable to induce a targeted modification of up to half the labels in complex multi-label classification scenarios, a capability that is out of reach of all the attacks proposed so far. As a further advantage, JMA requires very few iterations, thus resulting more efficient than existing methods.
\end{abstract}



\section{Introduction}
\label{sec:introduction}
%
%
%
%
%
%
%


\IEEEPARstart{A}{dversarial} examples, namely, quasi-imperceptible perturbations capable to induce an incorrect decision, are a serious threat to deep-learning  classifiers \cite{szegedy2013intriguing,goodfellow2014explaining}. Since the publication of the seminal work by Szegedy et al. \cite{szegedy2013intriguing} in which the existence of adversarial examples was first observed, a large number of gradient-based methods have been proposed to implement adversarial attacks against Deep Neural Networks (DNNs) in white-box and black-box scenarios \cite{goodfellow2014explaining, papernot2016JSMA, papernot,li2024adversarial}. 

Most attacks create the adversarial examples by minimizing a cost function subject to a constraint on the maximum perturbation introduced in the image.
Moreover, they focus on single-label classifiers based on one-hot encoding. In this setting,
the final activation layer consists in the application
of an activation function (usually a softmax) and a normalization, mapping the last layer outputs, called logits,
into a probability vector,
associating a probability value to
each class.
%
%
The loss function corresponds to the
categorical cross entropy. As a consequence,
in the targeted case, adversarial examples focus on raising the probability value of the target class.
%
This is obviously the best strategy
with single-label classification
however, this strategy is not optimal in general, e.g., in the presence of output encoding schemes based on channel coding, and in the case of multi-label classification.
A more flexible attack working at the logits level has been proposed by Carlini and Wagner  \cite{carlini2017towards}.
The attack works by decreasing the difference between the largest logit and the logit of the target class (in the targeted case).
Working at the logits level allows to avoid vanishing gradient problems \cite{hochreiter2001gradient}, hence Carlini and Wagner method often yields better performance compared to methods that directly minimize the loss term.
However, working on two logits at a time, without considering the effect of the perturbation on the other logits, is clearly suboptimum in general, with the consequence that the approach in \cite{carlini2017towards}  often results in a greedy, lengthy, iterative minimization process.
The  problem is  more evident when the attacker aims at attacking a DNN adopting different encoding mechanisms, like Error Correction Output Coding (ECOC) \cite{verma2019ecoc}, and also multi-label DNNs. In this case, all the logits have to be modified simultaneously.

Crafting adversarial examples for networks adopting a generic output encoding scheme,
and  for  multi-label classification, is a very challenging task, that, to the best of the authors' knowledge,
has not been given much attention, with the exception of some scattered works proposing suboptimum approaches  \cite{song2018multi, zhou2020generating}, \CH{\cite{Mahmood_2024_CVPR}}.
 From the attacker side, the difficulties of applying an optimal attack in this case are due to the fact
that the labels are not mutually exclusive (like in the one-
hot case), and decreasing or increasing the network output
in correspondence of some nodes may change the values
taken by the other nodes in an unpredictable way.

To overcome the above drawbacks, 
we introduce a new, theoretically sound, targeted adversarial attack,
named Jacobian-induced Mahalanobis distance Attack
(JMA).
JMA solves the original formulation of the adversarial example problem introduced in \cite{szegedy2013intriguing},
which aims at minimizing the strength of the adversarial perturbation subject to the constraint that the image is classified as belonging to the target class.
%
The solution is found via a two-step procedure: i) given a target point in the output space, the optimum perturbation moving the input sample to the target point is first determined, under a linear assumption on the effect of the distortion on the model output; ii) then,
%
%
the target point that minimizes the perturbation introduced by the attack is determined by  minimizing the Mahalanobis distance induced by the Jacobian matrix of the network input-output function.
By exploiting the Wolfe duality theorem, the problem is reduced to the solution of a Non-Negative Least Square (NNLS) problem, that can be efficiently solved numerically. The optimal perturbation derived in this way is then applied to the input image.
%
%
%
Ideally, JMA should produce the adversarial image in just one-shot. However, due to local optimality, in practice, it is sometimes necessary to carry out some iterations to obtain a valid adversarial example, recomputing the Jacobian matrix every time.

The results of the experiments confirm that the proposed attack is more efficient than state-of-the-art attacks from a computational perspective, requiring lower iterations to carry out the attack. State-of-the-art attacks are also outperformed in terms of distortion and attack success rate.
%
%
The effectiveness of the proposed adversarial attack method is particularly evident when it is used to attack networks adopting ECOC
and in particular in  multi-label classification scenarios, where JMA
is capable to simultaneously change in a desired way up to 10 out of 20 labels
of the output label vector, a capability which is out of reach
of the algorithms proposed so far.
%
We also verified the very good behavior of JMA
in the one-hot encoding scenario, where it  achieves  performance that are comparable to the state-of-the-art with a significant reduction of the computing time.
In summary, the main contributions of this paper are:
\begin{itemize}
\item We propose 
a new targeted adversarial attack for general classification frameworks,
named Jacobian-induced Mahalanobis distance Attack (JMA).
The attack resorts to the minimization of a Jacobian-induced Mahalanobis distance term, with the Jacobian matrix taking into account the effort (input space) to move the latent space representation of the input sample in a given direction.  In theory, the algorithm is one-shot.
\item We solve 
the constrained  minimization 
of the  Mahalanobis distance 
by exploiting the Wolfe duality theorem, reducing the problem to the solution of a non-negative least square (NNLS) problem, that can be solved numerically.
\item We validate
the new attack 
on various datasets (CIFAR-10, GTSRB, MNIST, VOC2012\CH{, MS-COCO, NUS-WIDE, and ImageNet}) and networks adopting different encoding schemes, namely  ECOC, multi-label, and one-hot encoding. \CH{Validation involves different types of architectures, including CNNs, Transformers and large visual models}.  The experiments confirm that the proposed attack is very general and can work in all these cases, being very efficient, often  requiring only very few iterations to attack an image.
\item 
We compare our new approach against 
several state-of-the-art attacks, showing that JMA is more efficient and requires less iterations.
In particular, in the multi-label classification scenario,  JMA is capable to change up to half labels of the output vector, even when the number of labels increases.
\end{itemize}

The rest of the paper is organized as follows: \CH{Section \ref{sec.relatedwork} gives an overview of the related work.
Section \ref{sec.background}} provides the main formalism and introduces the various classification frameworks. In Section \ref{sec.adv-attack}, the most relevant adversarial attacks \CH{that can be applied against
DNNs adopting output encoding} are reviewed.
Then, in Section \ref{sec.proposed}, the details of the JMA adversarial attack are described. The experimental methodology and setting are described in Section \ref{sec.methods}, while section \ref{sec.results} 
 reports and discusses the results of the experiments. The papers ends in Section \ref{sec.con} with some final remarks.

\section{\CH{Related works}}
\label{sec.relatedwork}

\CH{Adversarial attacks can be categorized in two main groups, namely white-box and black-box methods \cite{li2024adversarial}. }

\CH{Starting from Szegedy et al. \cite{szegedy2013intriguing} work, research on white-box attacks has mainly focused on the development  of gradient-based approaches 
that can reduce the complexity of the attack. Many greedy algorithms have been proposed that permit to find an effective adversarial example in a reasonable amount of time \cite{goodfellow2014explaining,papernot2016JSMA,kurakin2016physical,madry2017towards,moosavi2016deepfool,carlini2017towards}.
The Fast Sign Gradient Method (FGSM) method \cite{goodfellow2014explaining}
obtains an adversarial perturbation in a computationally efficient way by considering the sign of the gradient of the output with respect to the input image.
%
%
%
%
%
%
%
%
%
In \cite{kurakin2016physical}, an iterative version of FGSM is introduced by  applying  FGSM  multiple  times,  with  a  smaller  step  size,  each time by recomputing the gradient. This method is  often referred to as I-FGSM, or Basic Iterative Method (BIM).
Another attack similar to FGSM is the projected gradient  descent  (PGD) attack \cite{madry2017towards}, that can be regarded to as a multi-step extension of the FGSM attack where the clip operation performed by BIM on the gradient (to force the solution to stay in the [0,1] range) is replaced by gradient projection.  
Carlini and Wagner (C\&W) \cite{carlini2017towards} propose a more flexible method, working at the logits level, which can mitigate the gradient vanishing problem and improve performance in many cases.
Recently,  the AutoPGD and AutoAttack \cite{croce2020reliable} attacks have also been proposed. Such attacks improve PGD by designing an approach that automatically chooses the most suitable step size and perturbation size (in the case of AutoAttack) at every attack iteration, based on the behaviour of the objective function, and by using a different loss function, which improves performance and reduces the gradient vanishing problem.
Among the other  white-box gradient-based attacks it is worth mentioning
the Jacobian-based Saliency Map Attack (JSMA) \cite{papernot2016JSMA} 
%
%
and the DeepFool attack \cite{moosavi2016deepfool}.
JSMA \cite{papernot2016JSMA}  consists of a greedy iterative procedure which relies on forward
propagation to compute, at each iteration, a saliency map, indicating the
pixels that contribute most to the classification, while
DeepFool \cite{moosavi2016deepfool}  is an efficient iterative attack that considers  the  minimal perturbation with respect to a linearized classifier, stopping the attack when the boundary is crossed.}

\CH{A large segment of recent literature has also focused on the development of attacks that can work in real-world settings where the exact victim model is unknown. This includes transfer-based attacks, e.g.  \cite{papernot, weng2023logit, yang2024quantization}, which retain part of their effectiveness even against DNN models other than the one targeted by the attack, as well as black-box attacks operating under the assumption that the victim model can be queried a limited number of times \cite{queryTIFS23,zhang2024perception}.
}

\CH{All the above methods focus on single-label classifiers
adopting one-hot encoding schemes. The  problem of crafting adversarial attacks against classifiers adopting different output encoding schemes  is a challenging one that, to the best of the authors' knowledge, has not been much studied, with the exception of a few scattered works.  Song et. al
\cite{song2018multi} first addressed this problem and extended  the C\&W and the DeepFool attacks to a multi-label setting. These extended algorithms are referred as ML-C\&W and ML-DF.  Inspired by ML-DF, another approach
to implement an adversarial attack in the multi-label case, named Multi-Label Attack via Linear
programming (MLA-LP), has been proposed in \cite{zhou2020generating}. 
Recently, \cite{Mahmood_2024_CVPR} proposed a fundamentally different multi-label attack method that enforces semantic consistency across all predicted labels in the adversarial image. The approach introduced in \cite{Mahmood_2024_CVPR} leverages an efficient search algorithm over a knowledge graph that captures label dependencies.}
\CH{Beyond multi-label classification, a white-box attack tailored against DNNs adopting the ECOC encoding scheme has been proposed in \cite{zhang2020challenging}.}
\CH{A completely different approach, that can be applied to any network regardless of the encoding scheme, often leading to a large attack distortion in the input space,  is the Layerwise Origin-Target Synthesis (LOTS) attack \cite{rozsa2017LOTS}.}
\CH{Due to their relevance for our research, the attacks targeting classifiers adopting output encoding are presented in detail in Section \ref{sec.ECOC}.}

\section{Background and Notation}
\label{sec.background}





In this section, we review the main output encoding schemes for DNN classifiers.

Let $x \in \mathbb{R}^m$ denote the input of the network, and $y(x)$ (or simply $y$) the class $x$ belongs to. 
Le $l$ be the number of classes.
To classify $x$, the neural network first maps $x$ into a reduced dimensional space $\mathbb{R}^n$ ($n < m$). Every class is associated to an output label column vector $c_k = (c_{k1},c_{k2},...,c_{kn})^T$. Let  $f(x): \mathbb{R}^m \rightarrow \mathbb{R}^n $ be a column vector indicating the end-to-end neural network function, and $f_i(x)$ the $i$-th element of the vector.
Classification in favour of one of the $l$ classes is obtained by applying a function $\phi$ to $f(x)$, the exact form of $\phi$ depending on the output encoding scheme used by the network.
%
%
In the following, $z = (z_1,z_2,...,z_n)^T$ denotes the vector with the logit values,
that is, the values of the network nodes before the final activation function, which is responsible to map the output of the penultimate layer of the network into the [0,1] (sometimes [-1, 1]) range. Given an image $x$, the notation $f^{-j}(x)$ is used to refer to the internal model representation at layer $L-j$ ($L$ being the total number of layers of the network). With this notation, $z(x) = f^{-1}(x)$.



\subsection{One-hot encoding}
In the case of networks adopting the one-hot encoding scheme, the number of output nodes corresponds to the number of classes ($l = n$), and  $c_k$ is a binary vector with all 0's except for position $k$, where it takes value 1.
In this case, the length of
$z$ is $l$, and the logits are directly mapped onto the output nodes through a
softmax function as follows:
\begin{equation}\label{eq:softmax}
f_k(x) =  \frac{\exp(z_k)}{\sum_{i=1}^{l}\exp(z_i)},
\end{equation}
for $k = 1,..,l$.
This allows to interpret $f_k(x)$ as the probability assigned to class
$k$,
and the final prediction is made by letting $\phi(x) = {\arg\max}_{k} \hspace{0.1cm} f_k(x)$.


Typically, in the one-hot encoding case, training is carried out by minimizing the categorical cross-entropy loss defined as $\mathcal{L}(x,y) = - \sum_{i = 1}^l c_{yi} \log(f_i(x)) = - \log (f_y(x))$.


\subsection{Error-correction-output-coding (ECOC)}

Sometimes the output classes are encoded by using the codewords of a channel code. In this way, error correction can be applied to recover from incorrect network behaviors. In this case,  the network output dimension $n$ corresponds to the length of the codewords
and $c_k$ is the codeword associated to the $k$-th output class. The number of classes is typically  less than $2^n$ ($l < 2^n$).
An example of the use of channel coding to define the class label vectors is given in \cite{dietterich1994solving} (ECOC).
The use of ECOC has also been proposed as a way to improve the robustness against  adversarial attacks in a white-box
setting \cite{verma2019ecoc}.
%
The rationale is the following: while with classifiers based on standard one-hot encoding the
attacker can induce an error by modifying one single logit
(the one associated to the target class), the final decision of
the ECOC classifier depends on multiple logits in a complicated
manner, and hence it is supposedly more difficult to
attack.
%
%

Formally, with ECOC, a distinct codeword $c_k$ is assigned to every output class.
Let $C= \{c_1,c_2,\cdots, c_l\}$ define the codebook, that is, the matrix of codewords.
Each element of  $C$ can take values in  $\{-1,1\}$.
%
To compute the output of the network the logits are mapped into the $[-1, 1]$ range by means of an activation function $\sigma()$, that is $f(x) = \sigma(z)$, where $\sigma()$ is  applied element-wise to the logits. A common choice for $\sigma()$ is the $\tanh$ function.
To make the final decision, the probability of class $k$ is first computed:
\begin{equation}\label{eq:ecoc_map}
p_k(x) = \frac{\max(f(x)^T c_k, 0)}{\sum_{i=1}^{l} \max(f(x)^T c_i,0)},
\end{equation}
where $f(x)^T c_k$ is the inner product between $f(x)$ and $c_k$.
Since $c_{ij}$'s take values in  $\{-1,1\}$, the {\em max} is necessary to avoid negative probabilities.
Then,  the model's final prediction is given by $\phi(x) = {\arg\max}_{k} \hspace{0.1cm} p_k(x)$.
Note that Eq. \eqref{eq:ecoc_map} reduces to Eq. \eqref{eq:softmax} in the case of one-hot encoding, when ${C} = \mathbf{I}_{l}$ and where $\mathbf{I}_{l}$ is the identity $l \times l$ matrix.
In the ECOC case, training is usually carried out by minimizing the hinge loss function, defined as $\mathcal{L}(x,y) = \sum_{i=1}^n \max(1- c_{yi}  f_i(x), 0)$.

\subsection{Multi-label classification}
\label{sec.multi-label}

Multi-label classification is the scenario wherein the classifier is asked to decide about the presence or absence within the image of $n$ image characteristics or features. For example, the classifier may be asked to detect the presence of $n$ possible classes of objects, or decide about $n$ binary attributes (like indoor/outdoor, night/day, sunny/rainy). The presence/absence of the looked-for features is encoded by the $n$ outputs of the network
(see, for instance, \cite{Devil, wang2016cnn}). In general, the number of possible outputs is $l = 2^n$, and each output is encoded in a matrix $C$ having size $n \times 2^n$. In fact, it is possible that some labels' combinations are not feasible. In such a case, $C$ contains only the allowed combinations, somewhat playing the same role of channel coding. 
In this paper, multi-label classification indicates a situation where all combinations of labels are possible.

In the case of multi-label classification, given the logit vector $z$, the activation function $\sigma()$ is applied element-wise to the components of $z$, and the  prediction on each label component is made component-wise.
Assuming that  $c_k \in \{0,1\}^n$ for any $k$, and that the logistic function is applied to the logits, decoding corresponds to 0.5-thresholding each output score independently\footnote{If $c_k \in \{-1,1\}^n$,   tanh activation  followed by 0-thresholding is applied.}.
%
For the loss function, a common choice is the (multi-label) binary cross-entropy loss
$\mathcal{L}(x,y) = -  \sum_{i = 1}^n (c_{yi} \log(f_i(x)) + (1-c_{yi}) \log(1- f_{i}(x)))$.

\section{Adversarial attacks against DNNs} 
\label{sec.adv-attack}


The vulnerability of deep neural networks (DNNs) to adversarial samples has been first pointed out by Szegedy et al. in \cite{szegedy2013intriguing}. 
When the goal of the attacker is to induce a generic misclassification, \CH{the attack is referred to as untargeted}
\CH{Conversely, if the misclassification aims at a specific class, the attack is referred to as targeted.}
According to \cite{szegedy2013intriguing}, the generation of a targeted adversarial example can be formally described as
%
\begin{equation}\label{Eq:adv def}
	\begin{aligned}
		&\textrm{minimize} \quad ||\delta||_2 \\
		&\textrm{s.t.} \quad \phi(x+\delta) = t \neq \phi(x)\\
		&\textrm{and} \quad x+\delta \in [0,1]^m
	\end{aligned},
\end{equation}
where  $ ||\delta||_2 $ denotes the $ l_2 $ norm of the perturbation $ \delta $, and $ t $ is the target class of the attack. The same formulation holds for the untargeted case, with the first constraint replaced by $\phi(x+\delta) \neq \phi(x)$.
The constraint $ x+\delta \in [0,1]^m $ makes sure that the resulting adversarial example is a valid input.




\subsection{Basic adversarial attacks}
\label{sec.basic-attack}

Solving the minimization in \eqref{Eq:adv def} is generally hard, then the adversarial examples are determined by solving the simplified problem
where a functional $\mathcal L({x+\delta}, t) + \lambda \cdot \left\|  \delta \right\|_2^2$ is minimized,
    where $\mathcal L({x+\delta}, t)$ is the loss function (usually, the categorical cross entropy loss) under the target class $t$, and $\lambda$ is a parameter balancing the two terms. This problem  can be solved by the box-constraint L-BFGS method \cite{szegedy2013intriguing}. Specifically, a line search is carried out to find the value of $\lambda > 0$ for which the solution satisfies the adversarial condition, that is $\phi(x+\delta) = t$.
It is immediate to see that, in the one-hot case,  when the categorical cross entropy loss is considered,  only the target class contributes to the loss term, and then the attack only cares about increasing $f_t(x)$, regardless of the other scores, all the more that, due to the presence of the softmax, this also implies decreasing the other outputs.
To  reduce the  complexity of the L-BFGS attack, several suboptimal solutions have been proposed, considering the  problem of minimizing the loss function $\mathcal L$ subject to a constraint on the perturbation $\delta$,
\CH{like I-FGSM (or BIM) \cite{kurakin2016physical} and PGD \cite{madry2017towards}.}
A  flexible and efficient method working at the logits level has been proposed by
Carlini and Wagner \cite{carlini2017towards}.
In the C\&W attack in fact,  the  loss  term  is  replaced  by  a  properly  designed function  depending on the difference between the logit of the  target  class  and  the largest  among the logits of the  other  classes.
The perturbation is kept in the valid range by properly modifying the objective function.
In this way, the following unconstrained problem is considered and solved (for the case of $L_2$  metric):
    	\begin{align}\label{eq:CW}
       \min_{w} & \hspace{0.2cm} ||\delta(w)||_2^2 +  \nonumber\\
       &  \lambda \cdot \max(\max_{i \neq t} z_i(x + \delta(w)) - z_t(x + \delta(w)), - \kappa),
    	\end{align}
where $\delta(w) = 1/2 (\tanh(w) + 1) - x$ is the distortion introduced in the image, and the constant parameter $\kappa > 0$ is used to encourage the attack to generate a high confidence attacked image.
Since logits are more sensitive to modifications of the input than the probability values obtained after the softmax activation C\&W attack is less sensitive to vanishing gradient problems.
However,  by considering only two logits at time, and neglecting the effect of the perturbation on the other logits, this approach is particularly suited to the one-hot case and is highly suboptimum when other output encoding schemes are used.

\subsection{Attacks against DNNs with output encoding}
\label{sec.ECOC}

The adversarial attack algorithm and methods described above focus on models based on one-hot encoding. This is a favorable scenario for the attacker, who only needs to focus on increasing the output score (or logit) of the target class $t$. This is not the case with DNNs based on different output encoding schemes.
The simplest examples are multi-label classifiers.
From the attacker's side, the difficulties of carrying out the attack in this case are due to the fact that the labels are not {\em mutually exclusive}, and decreasing or increasing the network output in correspondence of some nodes may change the values taken by the other nodes in an unpredictable way. In addition, the attacker may want to modify more than one output label in a desired way. For instance, he may want to induce the classifier to interpret a {\em daylight} image showing a {\em car} driving in the {\em rain}, as a {\em night} image of a {\em car} driving with {\em no rain}.
This prevents the application of the basic adversarial algorithms described in the previous section.

\CH{As mentioned in Section \ref{sec.relatedwork} 
} the problem of crafting adversarial examples against  multi-label classifiers has been  addressed in \cite{song2018multi}, where the C\&W  and the DeepFool attacks have been extended to work in a multi-label setting.
The extension of C\&W attack, 
\CH{hereafter referred to as  ML-C\&W, }
works as follows: for every output node,  a hinge loss term  similar to  the one in Eq.~\eqref{eq:CW} is considered in the minimization.
%
More specifically, the following term is added to the minimization:  $\sum_{i=0}^{n}\max(0, \gamma-c_{ti} \cdot  z_i(x + \delta))$,
where $c_t$ is the target label vector and $\gamma$ is a confidence parameter.
%
%
%
%
Note that the extension implicitly assumes that decoding is carried out element-wise.
This is not true when all labels' combinations are possible and with schemes adopting output channel coding.
However, by working at the logits level, the  C\&W attack is  general and can also be extended to work  with different output encoding schemes (as discussed below).
%

\CH{The extension of the DeepFool attack (ML-DF) works as follows:
instead} of targeting an objective function minimizing the distortion, ML-DF looks for the minimum perturbation that allows to enter the target region under a linear assumption on the model behavior.
ML-DF works at the output score level, and hence is more prone to the vanishing gradient problem than ML-C\&W.
Given the threshold vector ${\rho} = ({\rho}_1, {\rho}_2,..,{\rho}_l)^T$
%
for the score level output ($\rho_i = 0.5$ when $f_i \in [0,1]$,  $0$  when $f_i \in [-1,1]$),  a vector with the distance to the boundary is derived and used to compute the perturbation.
The perturbation obtained in this way is applied to the image. ML-DF  applies the above procedure iteratively until the attack succeeds or the maximum number of iterations is reached.
%
%
 ML-DF is a  greedy method and often leads to very large distortions; Moreover,  the performance obtained
 %
 are always inferior  with respect to ML-C\&W \cite{song2018multi, zhou2020generating}.

\CH{MLA-LP, which also implements an adversarial attack for the multi-label case, works as follows.}
By assuming that the loss changes are linear for small distortions, a simplified formulation of the attack is solved to  minimize
the $L_{\infty}$ distortion introduced in the image subject to a
constraint on the  loss value, requiring that the final loss is lower than the threshold value.
Thanks to the adoption of the $L_{\infty}$ distortion, the problem can be easily solved by linear programming methods.
Similarly to ML-DF, MLA-LP works at the final score output level.
By considering the loss function instead of the model output function, MLA-LP tends to modify less the confidences of non-attacked labels with respect to ML-DF, with a lower distortion of the attacked image.

%
%
%
%

When the output coding scheme is based on ECOC, an adversarial attack can be carried out by extending to this setting the C\&W method.
Specifically, the C\&W attack can be applied at the probability level, exploiting the specific mapping of the logits to probabilities.
In particular, C\&W attack can be used to attack networks adopting ECOC output encoding by replacing the logit terms  $z_i(x)$ with $f(x)^T c_i$, that is, by considering the loss term $ \max(\max_{i \neq t} f(x + \delta(w))^T c_i - f(x + \delta(w))^T c_t, - \kappa)$.

In the case of an ECOC-based network,
adversarial examples can also be obtained by using the attack described in
  \cite{zhang2020challenging}  (in the following, this method is referred to as ECOC attack).
%
%
Such an attack incorporates within the minimization problem the ECOC decoding procedure.
Formally, the optimization problem solved in \cite{zhang2020challenging} is defined as:
%
%
    	\begin{align}\label{eq:ECOCattack}
       \min_{\delta} & \hspace{0.2cm} \left(||\delta||_2^2 -  \lambda \cdot \min_{i} (c_{ti} \cdot z_i(x + \delta), \eta)\right),
    	\end{align}
where  $\eta$ is a constant parameter setting a confidence threshold for the attack.
%


Although not specialized for this case, ML-C\&W  can also be used to attack ECOC-based networks. Instead,  ML-DeepFool and MLA-LP, that work at the score output level and  requires knowledge of the  threshold vector $\rho$, cannot be extended to the case of networks adopting an output encoding scheme like ECOC. In fact, the threshold vector  $\rho$ is not available in the ECOC case since the decision is made after the correlations are computed. A straightforward,  highly suboptimum, way to apply these methods to the ECOC case is by setting  $\rho = (0,0,...,0)^T$, as if all the codewords were possible, and performing the decoding element-wise.

\subsubsection{The LOTS attack}

A completely different kind of attack, which  can be applied to any network regardless of the output encoding scheme is  the Layerwise Origin-Target Synthesis (LOTS) attack \cite{rozsa2017LOTS}.
%
%
LOTS generates the adversarial example by modifying the input sample so that its representation in the feature space is as similar as possible to that of a given target sample. Formally, given the internal representation $f^{-j}(x)$  of an image $x$ at layer $L-j$
and a target internal representation $f^{-j}(x_t)$, LOTS attempts to minimize the term:
\begin{equation}
|| f^{-j}(x^t) - f^{-j}(x)||^2.
\end{equation}
At each iteration, the algorithm updates the input sample  as
%
\begin{equation}
x^{(i)} = x^{(i-1)} - \frac{\nu^{-j}(x, x^t)}{\max_x(|\nu^{-j}(x, x^t)|)}
\end{equation}
where $\nu^{-j}(x, x^t) = \nabla_x \left(|| f^{-j}(x^t) - f^{-j}(x)||^2\right)$, and $|\cdot|$ is applied  element-wise,
until the Euclidean distance between $f^{-l}(x)$ and the target is smaller than a predefined threshold.
The final perturbed image $x^{adv}$ obtained in this way has an internal representation at layer $L-j$ that mimics that of the target sample $x^t$.
In LOTS, the target point $x^t$ is any point belonging to the target class of the attack. 
It is worth noticing that in general $x^t$ may not be available, e.g. in the multilabel case when the attacker targets an arbitrary labels' combination.

When applied to the logits level ($f^{-1}$), LOTS simultaneously modifies all the logits. However, by focusing on the minimization of the Euclidean loss,
the distortion introduced  in the input space is not controlled.\\

\subsubsection{JMA and prior art}
In contrast to LOTS \cite{rozsa2017LOTS}, JMA defines the target point in the feature space that minimizes perturbation necessary to move the feature representation of the input sample to the target point.
It does so, by relying on the Mahalanobis distance induced by the Jacobian matrix of the neural network function so to take into account the different effort, in terms of input perturbation, required to move the sample in different directions. 
%
ML-DF and MLA-LP \cite{song2018multi, zhou2020generating} are also based on the Jacobian matrix, however,  in  \cite{song2018multi}, the Jacobian matrix is used to implement a greedy algorithm, while in \cite{zhou2020generating} is used to solve a suboptimum formulation of the attack. 
With JMA, instead,
the original formulation of the attack by Szegedy et al. is solved under a first-order approximation. 
With such an approximation, in fact, the solution of the attack problem can be reduced to the solution of a constrained quadratic programming problem (where the objective function is a Mahalanobis
distance term induced by the Jacobian).
Moreover, our algorithm is a general one and can work either at the output level or at the logits level, regardless of the output encoding scheme used.

\section{The JMA  attack}
\label{sec.proposed}

As mentioned in Section \ref{sec:introduction}, 
our method is designed to operate at the logits level.
In fact, while modifying one or two logits at a time
allows to carry out a close-to-optimal adversarial attack in the one-hot encoding case \cite{carlini2017towards}, in the more general case of DNNs based on output encoding,  this approach is highly suboptimum.
%
In this scenario, in addition to considering all the logits simultaneously, as done in the ML-C\&W attack, the correlation among the logits and the effort required in the input space to move the input sample in a given direction (in the logit or feature space) must be considered.
 JMA tackles the above problems and finds the optimal adversarial image, by solving the optimization problem in Eq. \eqref{Eq:adv def}, under a first-order approximation of the behavior of the network function.

In the following, we look at 
the output of the neural network function $f$ \footnote{For ease of notation, in the following, the model output is considered, however the same procedure can be applied at the logit level.}
  as a generic point in $\mathbb{R}^n$.
 According to Eq. \eqref{Eq:adv def},
 the optimum target point for the attacker corresponds to the point in the target decision region that can be reached by introducing  the {\em minimal} perturbation in the image $x$. Such a point is not necessarily the point in the target decision region closest to the current output point $f(x)$  in the Euclidean norm, since evaluating the distance in the output space does not take into account the effort necessary to move the input sample into the desired output point.
 %
Let $x_0$ be the to-be-attacked image and let us denote with $\delta$ the perturbation applied by the attacker.
Let $r = f(x_0) + d$ \CH{denote a generic  point in the output space}.
The goal of the attacker is to determine the displacement  $d$, in such a way that $r$ lies inside the decision region of the target class, and for which the Euclidean norm of the perturbation $\delta(d)$ necessary to reach $r$
is minimum, that is:
%
%
%
\begin{align}
\label{minimization-d}
		\min_{d: \phi(f(x_0) + d) = t} \left[ \min_{\substack{\delta : f(x_0 + \delta) - f(x_0) = d\\
												x_0 + \delta \in [0,1]^m}	} 					
		\hspace{0.2cm} ||\delta(d)||_2 \right].
\end{align}
\CH{The proposed attack works in two steps:
 \begin{list}{}{}
 	\item 1. 
    the optimum perturbation $\delta^*(d)$ moving the input sample to the (generic) target point $r$ is determined; 
 	\item 2.  we find the optimum target point $r^{*}$, i.e., the optimum displacement $d^*$ (where $r^* = f(x_0) + d^*$) by solving a constrained quadratic programming problem. The resulting $\delta^*(d^*)$ gives the minimum perturbation necessary to reach the target region.
 \end{list}
}
 
\CH{In the following, we first rewrite the minimization in a more convenient way, then we solve Step 1 and 2 under a first order approximation.}
Without loss of generality, we will assume that all the label codewords have the same norm, and that $\phi()$ applies a minimum distance decoding rule. In this case, forcing $\phi(f(x_0) + d) = t$ is equivalent to impose the following conditions:
%
\begin{equation}
 \label{expression-constraint}
r^T  c_t  \ge   r^T  c_i, \quad i = 1, 2 \dots l,
\end{equation}
%
that  can be rewritten as a function of $d$ as
%
\begin{equation}
d^T(c_i - c_t) \le g_{ti}, \quad i = 1, 2 \dots l,
\end{equation}
 where $g_{ti} = (f(x_0)^T c_t  -  f(x_0)^T c_i)$,
and hence the optimization problem in Eq. \eqref{minimization-d} can be rewritten as:
%
\begin{align}
\label{minimization-d2}
		\min_{\substack{d: d^T(c_i - c_t) \le g_{ti} \\ \forall i, i = 1,\cdots,l}} \bigg [ \min_{\delta : f(x_0 + \delta) - f(x_0) = d} 					
		\hspace{0.2cm} ||\delta(d)||_2 \bigg ].
\end{align}
Note that we neglected the constraint $x_0 + \delta \in [0,1]^m$, trusting that if $\delta$ is small the constraint is always satisfied. We will then  reconsider the effect of this constraint at the end of our derivation in Section \ref{sec.discussion}.

Under the assumption that the input perturbation is small (which should always be the case with adversarial examples), we can consider a first order approximation of the effect of the perturbation $\delta$ on the network output:
%
\begin{equation}
\label{eq.foapprox}
f(x_0 + \delta) \simeq f(x_0) +   J_{x_0} \delta,
\end{equation}
where $J_{x_0}$ denotes the Jacobian matrix of $f$ in $x_0$\footnote{$f$ can describe the output of the network or the logits (strictly speaking, in the latter case, $f$ should be replaced by $f^{-1}$)}:
\begin{equation}
J_x = \nabla_x f(x) |_{x_0}  =  \left[  \frac{\partial f_i(x_0)}{\partial x_j}\right]_{i = 1,\cdots,n, j=1,\cdots m}.
\end{equation}
%
%
%
\CH{We now show that the minimization in Eq.~\eqref{minimization-d2} can be solved under the first order approximation in Eq. \eqref{eq.foapprox}.
In particular, Step 1} corresponds to solving the following problem:
%
\begin{align}
\delta^{*}(d) = \arg \min_{\delta: f(x + \delta) - f(x_0) = d} ||\delta||_2,
\label{eq.optdelta}
\end{align}
that, under the linear approximation, becomes\footnote{To keep the notation light, here and in the sequel,  we omit the subscript of  the Jacobian matrix.}
\begin{align}
\label{min_ls}
\delta^{*}(d) = \arg \min_{\delta: d = J \delta} \|\delta\|_2,
\end{align}
that corresponds to finding the  minimum norm solution of the system of linear equations $d = J \delta$.
The solution of the above problem is given by the following proposition.
\begin{proposition}
Since $n < m$, under the assumption that $J$ has full rank\footnote{We verified this assumption experimentally as discussed in
Section  \ref{sec.attackParam}.}, the solution of the minimization in Eq. \eqref{min_ls} is:
%
%
\begin{align} \label{opt.attack}
{\delta}^{*}(d) = J^T(JJ^T)^{-1} d.
\end{align}
\end{proposition}
\begin{proof}
Eq. \eqref{min_ls} can be rewritten as the solution of the following non-negative least-squares  problem:
\begin{align}
\min_{\delta} \| J \delta - d||_2^2 = \min_{\delta} \left\{ \delta^T (J^T J )\delta - 2\delta^T J^T d + d^T d\right\}.
\end{align}
In fact, since  $(J^T J)$ is positive definite, the to-be-minimized function in the right-hand side of the equation is (strictly) convex and thus admits a minimum.
If $J$ has full rank, the solution of the above problem is unique and is given by \cite{boyd2004convex}:
\begin{equation}
\label{eq.solution}
\delta^* =  J^T (J J^T)^{-1} d,
\end{equation}
which satisfies  $J \delta^* = d$ (and hence $\| J \delta^* - d||_2^2 = 0$), thus also corresponding to the minimum norm solution in \eqref{min_ls}.
\end{proof}
From Proposition 1, we can express the norm of the optimum perturbation as follows
\begin{align} \label{expression-d}
||\delta^*(d)||^2 = &  \delta^{* T}  \delta^* =  (J^T {(J J^T)}^{-1} d)^T(J^T {(J J^T)}^{-1} d) \nonumber\\ = & d^T(J^T {(J J^T)}^{-1})^T (J^T {(J J^T)}^{-1} d) \nonumber\\ = & d^T {{(J J^T)}^{-1}}  d,
\end{align}
where in the last equality we have exploited the symmetry of $(J J^T)^{-1}$ (being the inverse of a symmetric matrix $(J J^T)$), and then ${((J J^T)^{-1})}^T = (J J^T)^{-1}$.

The optimal target point $r^*$, and hence the optimal distance term $d^*$, is then determined \CH{(Step 2)}  by solving the following minimization:
\begin{align}
\min_{d: d^T(c_i - c_t) \le g_{ti}, \forall i \ne t} &\hspace{0.2cm} d^T {{(J J^T)}^{-1}}  d.
\label{problem-v2}
\end{align}
%

Note that the above formulation corresponds to minimizing the Mahalanobis distance induced by the Jacobian matrix, between the current output of the network and the target point.
The above formulation has a very intuitive meaning, which is illustrated in Fig. \ref{fig.sketch}. The term $d^T {{(J J^T)}^{-1}}  d$  allows to take into account the effort required in the pixel domain to move along a given direction in the output domain.

\begin{figure}
\centering\includegraphics[width=0.9\columnwidth]{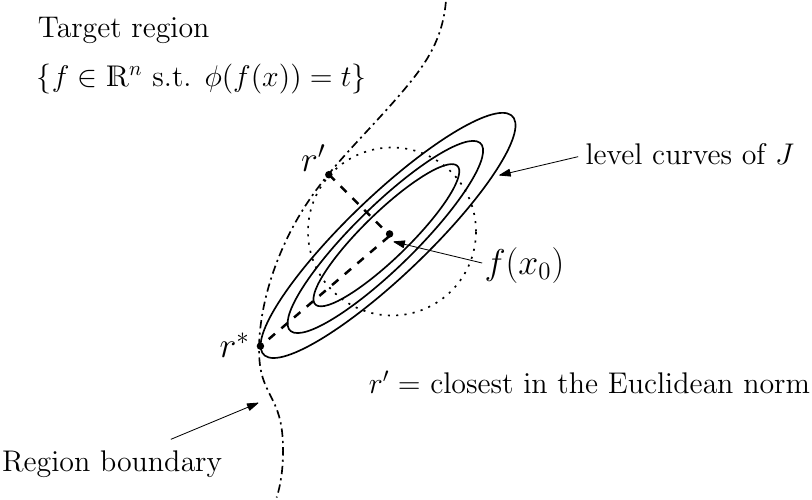}
\caption{Illustration of the intuition behind the formulation in \eqref{problem-v2}. 
	\CH{The point in the target region closest to $f(x_0)$ in the Euclidean norm  is $r'$. Then,  $r' = f(x_0) + d'$, where  $d'$ is the minimum norm displacement necessary to reach the target region. 
	  Instead, the point that can be reached with the minimum distortion in the input space, i.e., the point which minimizes $\delta$, is $r^{*}$ (this point corresponds to a displacement $d^*$ having larger norm than $d'$).  $r^*$ is the optimal target point that the attack wants to reach. }
	}
\label{fig.sketch}
\end{figure}
To solve \eqref{problem-v2}, we find convenient to rewrite it as follows:
\begin{align}
\min_{d:  A d  \le b} \frac{1}{2} d^T {{(J J^T)}^{-1}}  d,
\label{problem-v2_2}
\end{align}
where $A = [(c_1 - c_t)^T \dots  (c_l - c_t)^T]$ has size $(l-1) \times n$, and $b$ is an \CH{$(l-1)$}-long vector defined as $b_i = g_{ti}$.
We observe that when $l > n + 1$ the system is overdetermined.


\begin{theorem}\label{prop1}
Problem \eqref{problem-v2_2} (and hence \eqref{problem-v2}) has a unique solution given by:
\begin{equation}
\label{d_opt-v2}
d^* = - (J J^T) (A^T \lambda^*),
\end{equation}
where
\begin{equation}
\label{lambda_constraint}
\lambda^* = \arg \min_{\lambda} \hspace{0.2cm} \frac{1}{2} {\lambda}^T A (J J^T) A^T \lambda  + b^T \lambda, \quad \lambda \ge 0,
\end{equation}
%
and then
\begin{equation}
\label{sol-v3}
\delta^* = J^T (J J^T)^{-1} d^* = - J^T (A^T \lambda^*).
\end{equation}
\end{theorem}

\begin{proof}
%
 %
The optimal solution of the quadratic problem in  \eqref{problem-v2_2}
can be obtained by solving the easier Wolfe dual problem \cite{bertsekas1997nonlinear} defined as:
\begin{align}
\max_{d, \lambda} \hspace{0.2cm} &\frac{1}{2} d^T {{(J J^T)}^{-1}}  d + \lambda^T (A d - b)\nonumber\\
\textrm{s.t.} \hspace{0.2cm}  &{{(J J^T)}^{-1}}  d + A^T \lambda = 0\nonumber\\
&\lambda \ge 0,
\end{align}
where $\lambda$ is a column vector with $l-1$ entries.
By rewriting the objective function as $- \frac{1}{2} d^T {{(J J^T)}^{-1}} d +  d^T ({{(J J^T)}^{-1}} d  + A^T \lambda) - b^T \lambda$ and after some algebra, the optimization problem can be rephrased as:
%
\begin{align}
\min_{d, \lambda} \hspace{0.2cm} & \frac{1}{2} d^T {{(J J^T)}^{-1}}  d + b^T\lambda\nonumber\\
\textrm{s.t.} \hspace{0.2cm} & {{(J J^T)}^{-1}}  d + A^T \lambda = 0\nonumber\\
& \lambda \ge 0.
\end{align}
Solving the equality constraint yields $d^*$ in \eqref{d_opt-v2} as a function of $\lambda$. By substituting $d^*$ in the objective  function and exploiting the symmetry of $(J J^T)$ we obtain \eqref{lambda_constraint}.
\end{proof}

The problem in \eqref{lambda_constraint} is a NNLS problem \cite{bertsekas1997nonlinear,boyd2004convex}, for which several numerical solvers exist \cite{fletcher2013practical}, many of them belonging to the class of active set methods, see for instance \cite{lawson1995solving,MYRE2017755}. In particular, an easy-to-implement algorithm, whose complexity grows linearly with the number of label vectors, named Sequential Coordinate-Wise algorithm, has been proposed in \cite{NNLS}. \BT{Theorem 1 identifies the minimum perturbation necessary to enter the target region $\delta^*(d^*)$ (simply indicated as $\delta^*$ in the following). Then the adversarial image is computed as $x_{adv} = x + \delta^*$.}

Note that in the multi-label case, the number of label vectors grows exponentially with $n$ (since $l = 2^n$). However, in this case, the problem can be significantly simplified, as discussed in Section \ref{multi-label-simplification}.

\subsection{Iterative formulation of JMA}
\label{sec.discussion}

In principle, JMA should produce an adversarial example in one-shot. In practice, however, this is not always the case, due to the fact that the linear assumption of $f$ holds only in a small neighborhood of the input $x$, hence the assumption may not be met when the distortion necessary to attack the image is larger. In order to mitigate this problem, we consider the following iterative version of JMA:
\begin{itemize}
  \item When  $\delta^*$  does not bring $x$ into the target region, that is, $\phi(x + \delta^*) \neq t$, we update the input in the direction given by the perturbation considering a small step $\epsilon$, thus remaining in the close vicinity of $x$.
  \item When the target class is reached, a binary search is performed between the perturbed sample $ x+\delta^* $ and the original input $ x$, and the adversarial example resulting in the smallest perturbation is considered.
\end{itemize}

Finally, in practice, after that the optimal perturbation  $\delta^*$  is superimposed to $x$,
a clipping operation is performed to be sure that the solution remains in the  $[0,1]^m$ range.

The iterative version of JMA resulting from the application of the above steps is described in Algorithm 1. In the algorithm, the maximum number of iterations (updates) is set to $n_{it,max}$.


\begin{algorithm}[h!]
	\caption{Jacobian-induced Mahalanobis-distance attack}
	\label{ALG:attack algorithm}
	\begin{algorithmic}[1]
		\REQUIRE ~~\\ 
		max no. of iterations  $n_{it,max}$, to be attacked image $x$, target class $t$ (target label vector $ c_t $), updating step size $ \epsilon $
		\ENSURE ~~\\ 
		adversarial example $ x ^{adv}$
		\FOR{ $i \in [1, n_{it,max}]$}
		\STATE calculate $f(x)$
		\STATE calculate the Jacobian matrix by backward propagation: $ J= \nabla_x f(x)$
		\STATE calculate matrix $ A $ and vector $ b $ 
\STATE calculate $\lambda^*$ in Eq.\eqref{lambda_constraint} via the sequential Coordinate-Wise algorithm in \cite{NNLS}
		\STATE calculate adversarial perturbation: 
$\delta^* = - J^T (A^T \lambda^*)$.
		
		\STATE  $ x = x + \delta^* $ 
\STATE  $ x = \text{clip}(x , 0, 1 )$ 
		\IF{$\phi(f(x)) = t$} 
        \WHILE{$\phi(f(x)) = t$}
		      \STATE do binary search between $ x $ and $ x-\delta^* $
\STATE  $ x = \text{clip}(x , 0, 1 )$ 
        \ENDWHILE
        \STATE return $x^{adv}$
		\ELSE
           \STATE $ x =   (x - \delta^* ) + \epsilon\frac{\delta^*}{||\delta^*||} $
		\ENDIF
		
		\ENDFOR
	\end{algorithmic}
\end{algorithm}

\subsection{Simplified formulation for the multi-label case}
\label{multi-label-simplification}


In most cases,  the number of classes $l$, determining the number of rows of matrix $A$ in Eq. \eqref{problem-v2_2}, is limited.
For multi-label classification, however,
all labels' combinations are possible,
and $l = 2^n$. Therefore,
the size of $A$ increases exponentially slowing down the  attack.
Luckily, in this case, the  minimization problem  can be significantly simplified 
by properly rephrasing  the constraint in  \eqref{expression-constraint}.
Specifically, since all the codeword combinations are possible,  the constraint forcing the solution to lye in the
desired region can be rewritten element-wise as follows:
\begin{equation}
 \label{expression-constraint-ML}
f_j  \cdot c_{tj}  \ge   0, \quad j = 1, 2 \dots n.
\end{equation}
By rewriting  \eqref{expression-constraint-ML} as a function of $d$, we get
\begin{equation}
 \label{expression-constraint-ML-2}
d_j \cdot c_{tj} \ge - f_j(x_0) \cdot c_{tj}, \quad j = 1, 2 \dots n.
\end{equation}
%
that can be rewritten in the form $Md \le e$, where  $e$ is a $n$-length vector with elements $e_j =  f(x_0)_j \cdot c_{tj}$ and $M$ is an $n \times n$ diagonal matrix with diagonal elements $M_{jj}  = - c_{tj}$.
Then, the minimization problem has the same form of \eqref{problem-v2_2}, with $M$ and $e$ replacing $A$ and $b$, and
Proposition \ref{prop1} remains valid, with the difference that the scalar vector $\lambda$ has now dimension $n$. Therefore, the NNLS problem can be solved by means of the sequential Coordinate-Wise algorithm in \cite{NNLS} with complexity $O(n)$ instead of $O(2^n)$ as in the original formulation.

\section{Experimental Setting}
\label{sec.methods}

To evaluate the performance of JMA, we have
run several experiments addressing different classification scenarios, including ECOC-based classification \cite{verma2019ecoc} and  multi-label classification. Although less significant for this study, we have also run some experiments in the one-hot encoding scenario.
In all the cases, JMA is applied at the logits level.

We implemented JMA
by using Python 3.6.9
via the Keras 2.3.1 API. We run the experiments by using an NVIDIA GeForce RTX 2080 Ti GPU. The code,
as well as the trained models
and  the information for the reproducibility of the experiments are publicly available  at \href{https://github.com/guowei-cn/JMA--A-General-Close-to-Optimal-Targeted-Adversarial-Attack-with-Improved-Efficiency.git}{https://github.com/guowei-cn/JMA--A-General-Close-to-Optimal-Targeted-Adversarial-Attack-with-Improved-Efficiency.git}.
In the following sections, we describe  the classification tasks considered in the experiments, present the evaluation and  comparison methodologies, and detail the experimental setting.
%


\subsection{Classification networks and settings}

\subsubsection{One-hot encoding} 

 %
%
%
To assess the performance of JMA in the single-label classification scenario,
we considered the task of traffic sign classification on the German Traffic Sign Recognition Benchmark (GTSRB) dataset \cite{stallkamp2012GTSRB}, with a VGG16 architecture. In this scenario, 
the accuracy of the the trained model on clean images is 0.995.
\CH{Moreover, to validate the effectiveness of the JMA algorithm in a challenging classification scenario with a large number of classes, we defined a new classification task by selecting 2,000 classes from the most populated categories in the ImageNet21K dataset \cite{ImageNet21K} (referred to as ImageNet2K in the following). A  ViT-B/16 network pretrained on ImageNet21K and fine-tuned on a training subset of ImageNet2K was used for this task.
The accuracy on clean images is 0.66.
}
%
%
%

%

\subsubsection{ECOC-based classifier}

We considered the ECOC framework for three image classification tasks,
namely MNIST \cite{lecun1998mnist}, CIFAR-10 \cite{krizhevsky2009cifar10}, and
GTSRB \cite{stallkamp2012GTSRB}.
We implemented the ECOC scheme
considering an
ensemble of networks, each one outputting
one bit\footnote{According to \cite{verma2019ecoc}, resorting to an ensemble of  networks  permits to achieve better
robustness against attacks.}. The number of branches $h$ of the ensemble is  $ h= 10 $ for MNIST and CIFAR-10, and $h=16 $ for GTSRB.
We used Hadamard codewords as suggested in \cite{verma2019ecoc}.
We considered an Hadamard code with $n = 16$ for MNIST and CIFAR-10. For the GTSRB case, we set $l = 32$, by selecting the classes with more examples, and used an Hadamart code with $n = 32$.
%
Following the original ECOC design, we considered the VGG16 architecture \cite{tfweights}
as baseline, with the
first 6 convolutional layers shared by all the networks of the ensemble (shared bottom),
and the remaining 10 layers (the last 8 convolutional layers
and the 2 fully connected layers) trained separately for
each ensemble branch. The weights of the ImageNet pre-trained model are used for the shared bottom.
%

\subsubsection{Multi-label classification}

%

We primarily tested the performance of JMA in a multi-label setting by using the VOC2012 dataset  \cite{VOC2012}, which is a benchmark dataset adopted in several multi-class classification works \cite{wang2016cnn,Devil}.  This dataset has been used to train models to recognize object inside scenes. Specifically, the 11,530 images of the dataset
contain objects from 20 classes.
%
Then, $n = 20$ and the number of possible label vectors is $2^n \approx 10^6$.
The dataset is split into training, validation, and testing subsets, with proportion 6:2:2.
We considered the same model architecture adopted in \cite{song2018multi}
that is, a standard InceptionV3 \cite{szegedy2016rethinking}.
The categorical hinge loss was used to train the model, with a standard Adam optimizer with $lr=10^{-4}$ and batch size 64.
The performance of the network in this case are measured in terms of mean average precision (mAP), that is, the mean of the average precision (AP)  for all the classes, where AP is a measure of the area under the precision-recall curve.
%
According to our experiments, the mAP of our model is 0.93, which is aligned with  \cite{song2018multi}.

\CH{To better show the generality of the proposed attack, we also trained additional models using different and more modern architectures, namely, ResNet50 \cite{he2016deep}, ViT-B/16 \cite{dosovitskiy2020image}, and CLIP+MLP \cite{radford2021learning} , and ran attacks against them 
The results regarding these networks are reported in Section \ref{sec.furtherExp}.}
\CH{
	We further validated the effectiveness of JMA on multi-label classification by running some additional experiments
	on more complex datasets, namely MS-COCO  \cite{lin2014microsoft}  and NUS-WIDE  \cite{NUS-WIDE} (see again Section \ref{sec.furtherExp})}.

\subsection{Comparison with the state-of-the-art}

We have compared the performance of JMA with the most relevant state-of-the-art attacks
for the various classification scenarios.
In all the cases, we used the code made available by the authors in public repositories.

For the multi-label case,  we considered  ML-C\&W \cite{song2018multi},
 MLA-LP \cite{zhou2020generating} , LOTS \cite{rozsa2017LOTS}  \CH{and the method in  \cite{Mahmood_2024_CVPR}, referred to as SemA-ML in the following.}
We did not consider ML-DF, since its performance are always inferior to those of ML-C\&W \cite{song2018multi,zhou2020generating})

Regarding LOTS, as described in Section \ref{sec.adv-attack}, it can be applied to any internal layer of the network.
In the experiments, for a fair comparison, we have applied LOTS at the logits layer.
By following \cite{rozsa2017LOTS}, for every target label vector,  we randomly selected 20 images
with that label vector, and averaged the logits  to get the target output vector (for target label vectors for which the total number of images available is lower than 20, we used all of them).
We stress that, in the way it works, LOTS can not target an arbitrary labels' combination, since examples for that combination may not be available.

For the case of ECOC-based classification, the comparison is carried out against C\&W,
which is also the attack considered in \cite{verma2019ecoc},
and   \cite{zhang2020challenging}  (ECOC Attack) specialized to work with the ECOC model  \cite{zhang2020challenging}.
Although suboptimal in this case, we also considered ML-C\&W and MLA-LP, by applying the latter as discussed in Section \ref{sec.ECOC}. However, given the complexity of the attack in the ECOC scenario and the high sub-optimality of MLA-LP, the performance of  this attack are extremely poor, and no image can be attacked by using it.
%
Finally, the performance of  LOTS are also assessed. In this case, the  target output vectors are computed by averaging 50 images (when the number of  available images is lower than 50, we averaged all of them).
%

For the one-hot encoding scenario, we considered the original C\&W algorithm  \cite{carlini2017towards}, which  is  one of the best performing  white-box attack working at the logits level.

We did not consider attacks like I-FGSM, PGD, and also AutoPGD and AutoAttack, 
as they cannot be
%
applied to frameworks different than single-label classification (see discussion in Section \ref{sec.relatedwork}),
e.g. to ECOC and in particular to the multi-label case, for which they would need a suitable extension via the definition of a proper loss function.


\subsection{Evaluation  methodology}

To test the  attacks in the various scenarios, we proceeded as follows:  we randomly picked 200 images from the test set, among those that are correctly classified by the network model \CH{(in the complex multi-label classification tasks of MS-COCO and NUS-WIDE, 
some label errors always occur, hence the attacked samples are chosen at random)};
then, for each image, we randomly picked a target label vector, different from the true label vector.

For multi-label classification,
in principle, there can be  - and in fact there are - combinations of labels that are not represented in the training set.
it is also possible that some of these combinations  do not correspond to valid  label vectors.
Since LOTS requires the availability of samples belonging to the target label vector, for ease of comparison, in the experiments, we considered label vectors that appear in the training set as target label vectors.
%
For the other methods, some tests have also been carried out  in the more challenging case where
the target  label vectors  correspond to randomly chosen target vectors. Such target vectors are obtained by randomly changing a prescribed number of bits in the label vector of the to-be-attacked image.

An attack is considered successful only when the predicted label vector and the target label vector coincide. If only a subset of the labels can be modified by the attack, and  the target label vector is not reached, we mark this as a failure.
We measured the performance of the attacks  in terms of Attack Success Rate (ASR), namely, the percentage of generated adversarial examples that are assigned to the target class. To measure the quality of the adversarial image,
we considered the Mean Square Error (MSE) of the attack, evaluated as MSE $= ||\delta||_2/\sqrt{H\times W\times 3}$, where $H\times W \times 3$ is the size of the image (image values are  in the [0,1] range).

%
For the tests with  random choice of the target vector,  we also report the average percentage of labels/bits successfully modified by the attack, indicated as bASR.
%
%

The computational complexity is measured by providing the average number of iterations required by the attack, and the time  - in seconds - necessary to run it.
In the following, we let $n_{it}$ denote the number of iterations
necessary to run an attack, that is, 
the number of
attack updates. The average number of iterations is denoted by $\bar{n}_{it}$.

\subsection{Attack Parameters Setting}
 \label{sec.attackParam}
%
The tunable parameters of JMA are the maximum number of iterations (${n}_{it, max}$)
and the step size $\epsilon$.
The maximum number of iterations
takes also into account the number of `for' loops and the number of steps of the binary search in the final iteration (see Algorithm 1), the total number of updates of JMA, then, is ${n}_{it} = (v-1) + n_{bs}$, where $v$ is the number of `for' loops of the algorithm ($v < {n}_{it, max}$, see Algorithm 1) and  $n_{bs}$ is the number of steps of the binary search.
$n_{bs}$ is set to  6 in our the experiments.

%

For C\&W, ML-C\&W and ECOC attack, the parameters are the initial  $\lambda$, the number of steps of the binary search $n_{bs}$,
and the maximum number of iterations $n_{it, max}$.
%
%
For all these attacks, a binary search  is carried out over the  loss tradeoff parameter  $\lambda$   (see Section \ref{sec.basic-attack}).
At each binary search step, iterations are run (until $n_{it,max}$) to obtain the adversarial example. The number of iterations that we report considers only  the number of attack updates corresponding to the final $\lambda$ value used by the attack.
Hence, in these cases, the actual complexity of the attack is better reflected by the attack time.
As for the setting of these parameters, we consider  $n_{bs} = 5$ and $10$, and various values of the initial $\lambda$.
\CH{For SemA-ML, we set  $n_{it, max}=300$ following the official code,
with the other settings left as the default.}
For MLA-LP  and LOTS, the only parameter is
the maximum number of iterations $n_{it,max}$ allowed in the gradient
descent.
In the experiments,
we considered various values of $n_{it, max}$.
%
%
%
For  MLA-LP, we set  $n_{it, max}  = 100$ for the multi-label classification case, and  raised it up to $3000$ in the ECOC case when the attack is more difficult.
For LOTS, we set $n_{it, max}  = 2000$, namely the  maximum number considered in \cite{rozsa2017LOTS},  maximizing the chances to find an adversarial example.
Given that the time complexity of LOTS is limited compared to the other methods, in fact,
the time complexity of the attack remains small also for large values of $n_{it, max}$.
%
%
We have verified experimentally that LOTS converges within
1000 iterations 92\% of the times\footnote{Convergence is determined by checking whether the new loss value is close
enough to the average loss value of the last 10 iterations.},
validating the soundness of our choice.
For ML-C\&W, we set  $n_{it,max}  = 1000$ as default value,
and in most of the cases results are reported for $(\lambda, n_{bs}, n_{it, max}) = (0.01, 10, 1000)$, that resulted in the highest ASR.
Finally, for JMA, we generally got no benefit by increasing the number of iterations beyond 200, since most of the time JMA converges in by far less iterations, and 200 is more than enough to get and adversarial image in all the cases.
%

To simplify the comparison, the confidence parameter is set to 0 for all the attacks.
%
%
%
%
%
To avoid that the an attack gets stuck with some images (e.g.  in the multi-label case when many labels have to be changed), we set a limit on the running time of the attack, and  considered the attack unsuccessful when an adversarial example could not be generated in less than 1 hour.

Regarding JMA, we found experimentally  that, the assumption made in the theoretical analysis that the Jacobian matrix has full rank always holds when the number of iterations of the attack remains small. However, when the number of iterations increases
 it occasionally  happens that the rank of the Jacobian matrix is not full. In this case, an adversarial perturbation can not be found, resulting in a failure of the attack.

\section{Results} 
\label{sec.results}

The results we obtained in the various settings are reported in this section.
In all the tables,  $\bar{n}_{it}$  refers to the average number
of iterations/updating of the attack, averaged on the successful attacks.
The attack parameters reported are
$(\lambda, n_{bs}, n_{it,max})$ for C\&W, ML-C\&W and ECOC attack,     $(\epsilon,  n_{it,max})$  for JMA and $n_{it,max}$ for LOTS and MLA-LP.

\subsection{One-hot encoding                                                                                                                                                                                                                                                                                                                                                                                                                                                                                                                                                                                                                                                                                                                                                                                                                                                                                                                                                                                                                                                                                                                                                                                                                                                                                                                                                                                                                                                                                                                                                                                                                                                                                                                                                                                                                                                                                                                                                                                                                                                                                                                                                                                                                                                                                                                                                                                                                                                                                                                                                                                                                                                                                                                                                                                                                                                                                                                                                                                                                                                                        }

The main advantage of JMA with respect to state of the art is obtained for classifiers that do not adopt one-hot encoding.
In such a scenario, in fact, focusing only on the target logit or the final score, as done by the most common adversarial attacks,  is approximately optimal.
Nevertheless, we verified that JMA is also effective in the one-hot encoding scenario.

Table \ref{TAB:one-hot} reports the results on GTSRB.
We see that JMA  achieves a larger ASR compared to  C\&W, yet with a larger MSE.
This is due to the fact that JMA tends to modify all the output scores, while  C\&W directly focuses on raising the target output, which is an effective strategy in the one-hot encoding case.
In addition, JMA can obtain an adversarial image with much lower iterations, requiring only very few seconds (2.45 sec in the best setting) to attack an image.

\CH{The results of JMA against ImageNet2K are reported in Table \ref{TAB:one-hot_imagenet2k}.}
\CH{We observe that, although JMA is not specifically designed for this scenario, it can successfully attack all the images in very few iterations, with an MSE similar to that obtained on GTSRB. The attack time is larger, due to the higher complexity of the NNLS problem.}

\begin{table}[]
	\centering
	\caption{Results against one-hot encoding classification for  GTSRB.}
	\label{TAB:one-hot}
%
	\begin{tabular}{|l|c|c|c|c|c|}
		\hline
		& Parameters     & ASR   & MSE  & $\bar{n}_{it}$ & Time(s)\\
\hline
\ChangeRT{0.5pt}
		\multirow{5}{*}{JMA}
		&(0.05,   200) & 0.81 & 1.6e-4 & 50.65 & 7.44  \\  \cline{2-6}
		&(0.1,   200)  & 0.88 & 2.1e-4 & 38.00  & 5.66 \\  \cline{2-6}
		&(0.2,   200)  & 0.94  & 2.6e-4 & 29.63 & 4.41 \\  \cline{2-6}
		&(0.6,200)    & 0.98 & 4.3e-4 & 16.20 & 2.46 \\ \hline
\ChangeRT{0.5pt}
\multirow{5}{*}{C\&W} &(1e-4,5, 100)  & 0.19 & 2.2e-5  & 85.95 & 45.06 \\  \cline{2-6}
&(1e-4,5,200)   & 0.33 & 3.0e-5 & 171.88 & 90.11  \\  \cline{2-6}
&(1e-2,5, 500)  & 0.55  &  5.8e-5 & 443.18 & 233.45  \\  \cline{2-6}
&(1e-1,10, 2000) & 0.75  & 7.7e-5  & 1661.11 & 875.35  \\ \cline{2-6}
&(1e-1,10, 5000) & 0.82  & 7.9e-5  & 3577.5 & 1884.59  \\ \hline
	\end{tabular}
\end{table}

\begin{table}[]
	\centering
	\caption{\CH{Results of JMA against one-hot encoding classification with a large number of classes (ImageNet2K).}
    }
	\label{TAB:one-hot_imagenet2k}
    \CH{
	\begin{tabular}{|l|c|c|c|c|c|}
		\hline
		& Parameters     & ASR   & MSE  
        & $\bar{n}_{it}$ & Time(s) \\ \hline
            \ChangeRT{0.5pt}
		\multirow{5}{*}{JMA}
		&(0.1, 200)  & 0.99 & 3.8e-6 & 26.59 & 445.61  \\  \cline{2-6}
            &(0.2, 200)  & 1.00 & 6.2e-6 & 19.17 & 297.00 \\  \cline{2-6}
            &(0.6, 200)  & 1.00 & 2.1e-5 & 13.59 & 178.78 \\  \cline{2-6}
		&(1, 200)  & 1.00 & 4.2e-5 
        & 12.11 & 160.06 \\  \hline
	\end{tabular}
    }
\end{table}

\subsection{ECOC-based classification}

The results for ECOC-based classification on GTSRB, CIFAR-10, and MNIST are reported in Tables \ref{TAB:ECOC-GTSRB}-\ref{C4-TAB:ECOC-MNIST}.
%
%
In all the cases, JMA achieves a much higher ASR for a similar MSE, with a significantly lower computational complexity.
\CH{In the tables, NA stands for `Not Applicable',  when no image can be attacked (ASR = 0) and then the corresponding metric can not be evaluated.} 

Specifically, in the case of GTSRB (Table \ref{TAB:ECOC-GTSRB}), JMA achieves
ASR=0.98 with MSE = 1.5e-4, a total average number of iterations  $\bar{n}_{it} = 29.15$ and attack time of $9.41$ sec.
This represents a dramatic improvement with respect to the other methods. In the case of ECOC attack,
for instance, the best result is ASR=0.93,  reached for a similar distortion (MSE = 1.9e-4),  with $\bar{n}_{it}$ larger than 1000 ($\bar{n}_{it} = 1089.03$), which is two orders of magnitude larger than JMA. We also see that the C\&W method fails in this case, with the ASR being about 0.30 for comparable MSE values. Remind that the C\&W method is not designed to work in the ECOC case (the results are aligned with those reported in \cite{verma2019ecoc,zhang2020challenging}), and when applied to the ECOC classification scenario it loses part of its effectiveness (see discussion in Section \ref{sec.adv-attack}). ML-C\&W is more effective, however it can only achieve  ASR=0.78 with a significantly larger MSE (9.7e-4) and at the price of a larger  complexity.  A  very poor behavior of the attack  is observed for LOTS  and in particular MLA-LP, that is never able to attack the test images in 1 hour, notwithstanding the large value of $n_{it,max}$. Since this method implements a greedy approach, it is not surprising that it fails in the complex ECOC classification scenario (more in general, the performance of MLA-LP are poor whenever the attack aims at changing several bits/labels - see the results on multi-label classification).

In the case of
CIFAR-10 (Table \ref{C4-TAB:ECOC-CIFAR}), JMA achieves ASR = 1 with a pretty small distortion (MSE = 1.1e-4) and a very  low complexity.
C\&W  attack achieves a much lower ASR (ASR = 0.49)  for a slightly lower MSE, with higher complexity.
The ECOC attack can achieve ASR = 0.99 with MSE=1.9e-4
in the setting with $n_{it,max}= 2000$, in which case the computational complexity of the algorithm is  high ($\bar{n}_{it} = 1347.42$ for an average attack time of $557.86$ sec).
Regarding LOTS, it can achieve ASR = 1, yet with a larger MSE. The results with LOTS confirm the lower complexity of this method with respect to ECOC attack and  C\&W.
%

The effectiveness of JMA is also verified in the  MNIST case (see Table \ref{C4-TAB:ECOC-MNIST}).
%
In this case, JMA reaches ASR = 1 with MSE=2.3e-3, while the ASR of  ECOC attack (in the best possible setting) and LOTS is respectively 0.73 and 0.78 for a similar MSE. The computational time is higher for LOTS and much higher for the ECOC attack. Both C\&W and ML-C\&W  have poor performance with an ASR lower than 0.52 for a MSE similar to that obtained by JMA.


Fig. \ref{fig:ECOC} reports some examples of images (successfully) attacked  with the various methods for the three tasks.

\begin{figure}
  \centering
  \begin{tabular}{@{}c@{}}
    \includegraphics[width=1\linewidth]{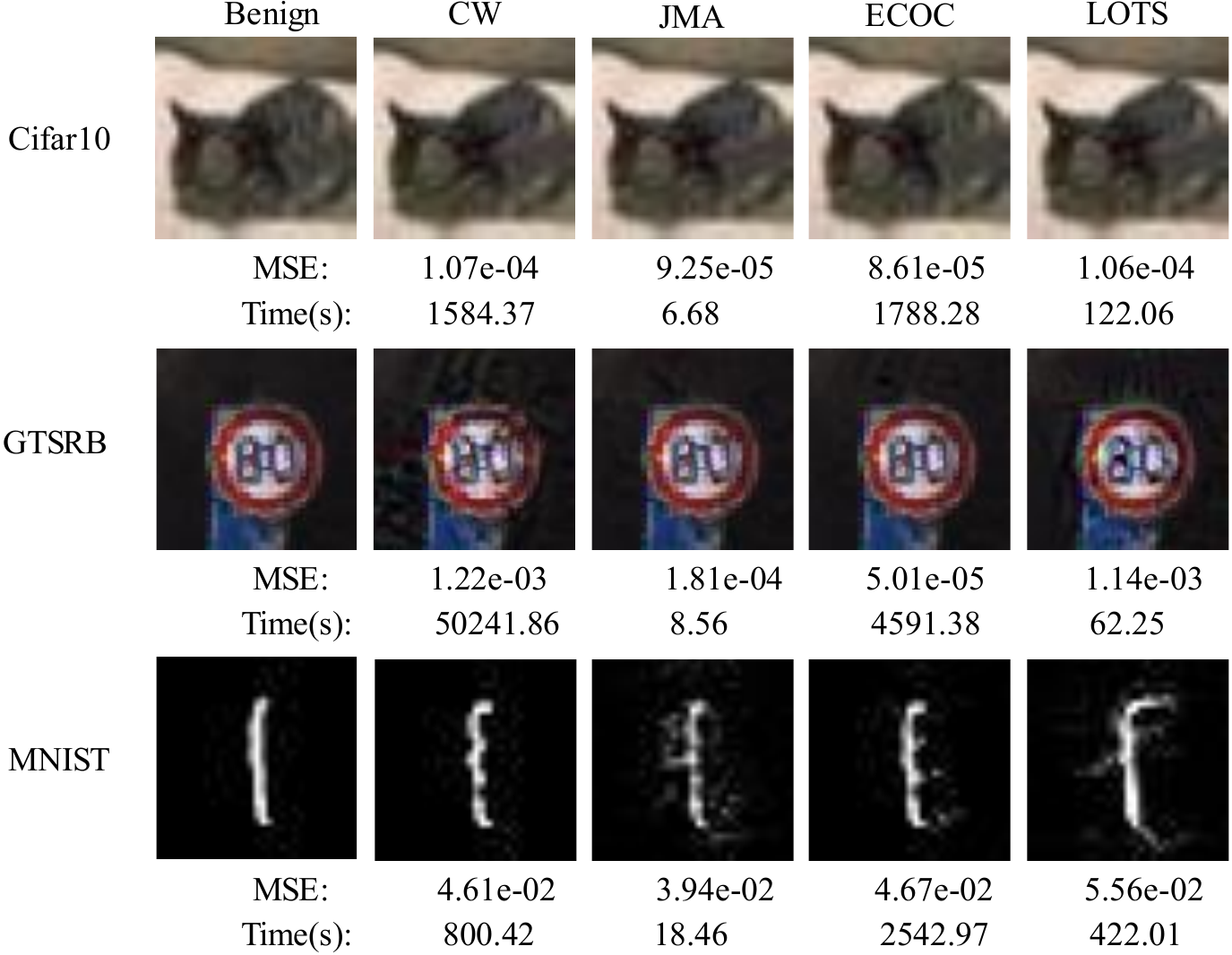}
  \end{tabular}

  \caption{Examples of attacked images in the ECOC classification case. The attacker's goal is to cause a misclassification from `cat' to `airplane' (CIFAR-10), from speed limitation `80kph' to `30kph' (GTSRB), and from digit `1' to `3' (MNIST). The distortion and time required for the attack are also provided for each case.}
  \label{fig:ECOC}
\end{figure}

\begin{table}
	\caption{Results against ECOC- classification on GTSRB. \CH{NA means `Not Applicable' (the metric can not be evaluated as no image can be attacked).}
}
    \label{TAB:ECOC-GTSRB}
    \centering
        \begin{tabular}{|l|c|c|c|c|c|}
            \hline
              & Parameters & ASR   & MSE & ${\scriptsize \bar{n}_{it}}$ & Time(s) \\ \hline
                          \multirow{3}{*}{JMA } 
            & (0.05, 200)            & 0.93  & 1.2e-4    & 58.46 & 18.14 \\ \cline{2-6}
            & (0.1, 200)           & 0.98  & 1.5e-4 & 29.15 & 9.41  \\ \cline{2-6} 
            & (0.2, 200)            & 0.99  & 2.0e-4  & 25.30 & 8.27 \\ \hline 
            \multirow{5}{*}{C\&W } & (1e-4, 5, 100)         & 0.08 & 1.2e-5   & 65.69  & 51.31 \\ \cline{2-6}
            & (1e-4, 5, 200)         & 0.11 & 1.6e-5  & 97.73  & 97.67 \\ \cline{2-6}
            & (1e-4, 5, 500)         & 0.12 & 1.8e-5  & 148.33 & 235.76 \\ \cline{2-6}
            & (1e-2, 5, 500)         & 0.21 & 6.1e-5 & 194.19 & 221.63 \\ \cline{2-6}
            & (1e-1, 10, 2000)       & 0.31 & 1.1e-4  & 782.29 & 1453.85 \\ \hline
            \multirow{5}{*}{\makecell{ECOC\\Attack}} & (1e-4, 5, 100)        & 0.35 & 4.5e-5 & 71.22 & 87.87 \\ \cline{2-6}
		& (1e-4, 5, 200)         & 0.40 & 3.8e-5  & 109.90 & 305.28 \\ \cline{2-6}
		& (1e-4, 5, 500)         & 0.43 & 3.6e-5  & 175.67 & 332.88 \\ \cline{2-6}
		& (1e-2, 5, 500)         & 0.61 & 8.6e-5  & 309.03 & 501.80 \\ \cline{2-6}
	    & (1e-1, 10, 2000)       & 0.93 & 1.9e-4  & 1089.03 & 2312.64 \\ \hline
                    ML-CW & (0.01, 10, 1000)   & 0.78 & 9.7e-4   & 111.02  & 2973.17 \\ \hline
            LOTS & $2000$   & 0.40 & 1.2e-4   & 395.33  & 97.52 \\ \hline
\ChangeRT{0.5pt}
            MLA-LP & $1000$   & 0.00 & NA  & NA  & ($>$ 1h) \\ \hline
        \end{tabular}
\end{table}

\begin{table}
	\caption{Results  against ECOC classification on  CIFAR10.}
	\label{C4-TAB:ECOC-CIFAR}
\centering
        \begin{tabular}{|l|c|c|c|c|c|}
            \hline
            & Parameters & ASR & MSE & $\bar{n}_{it}$ & Time(s) \\ \hline
            \ChangeRT{0.5pt}
            		\multirow{3}{*}{JMA } 
		& (0.05, 200)          & 1.00   & 1.1e-4    & 14.78 & 3.14 \\ \cline{2-6}
		& (0.1, 200)           & 1.00   & 1.3e-4 & 6.88  & 1.92 \\ \cline{2-6}
		& (0.5, 200)           & 1.00   & 1.5e-4 & 6.10  & 1.20 \\ \hline
\ChangeRT{0.5pt}
            \multirow{4}{*}{C\&W } & (1e-4, 5, 100)                 & 0.49  & 8.3e-5   & 40.76  & 23.20 \\ \cline{2-6}
            & (1e-4, 5, 500)                 & 0.60  & 7.2e-5 & 119.45 & 68.49 \\ \cline{2-6}
            & (1e-4, 10, 500)                & 0.96  & 1.7e-4  & 187.96 & 107.78\\ \cline{2-6}
            & (1e-1, 10, 2000)               & 0.99 & 1.6e-4  & 593.09 & 340.09 \\ \hline
\ChangeRT{0.5pt}
            \multirow{4}{*}{\makecell{ECOC\\Attack}}& (1e-4, 5, 100)                & 0.65 & 1.2e-4  & 83.08 & 50.36 \\ \cline{2-6}
            & (1e-4, 5, 500)                 & 0.92 & 1.4e-4 & 337.53 &  180.37 \\ \cline{2-6}
            & (1e-4, 10, 500)                & 0.95 & 1.8e-4   & 383.74 & 557.86 \\ \cline{2-6}
            & (1e-1, 10, 2000)               & 0.99 & 1.9e-4 & 1347.42 & 2173.31 \\ \hline
 \ChangeRT{0.5pt}
             ML-CW & (0.01, 10, 1000) & 0.565 &   1.0e-3 & 49.57  & 1409.02   \\ \hline
 \ChangeRT{0.5pt}
            LOTS & $2000$   & 1.00 & 2.6e-4  & 38.85   &  8.61 \\ \hline
 \ChangeRT{0.5pt}
            MLA-LP & $ 1000$   & 0.04  & 1.1e-4	 & 5 & 7.04 \\ \hline
        \end{tabular}
\end{table}

\begin{table}
	\caption{Results against ECOC classification on MNIST. 
}
	\label{C4-TAB:ECOC-MNIST}
	\centering
     %
            \begin{tabular}{|l|c|c|c|c|c|}
            \hline
            & Parameters & ASR   & MSE & $\bar{n}_{it}$ & Time(s) \\ \hline
            \ChangeRT{0.5pt}
                        \multirow{3}{*}{JMA } & (0.05, 200)  & 0.95 & 3.9e-3   & 48.61  & 196.55  \\ \cline{2-6}
            & (0.5, 200)   & 1.00 & 5.8e-3  & 12.04 & 45.53 \\ \cline{2-6}
            & (1, 200)     & 1.00 & 9.3e-3   & 8.69  & 33.58 \\ \hline
            \ChangeRT{0.5pt}
            \multirow{4}{*}{C\&W } & (1e-3, 10, 100)        & 0.01  & 2.2e-3    & 53.00 &  48.96 \\ \cline{2-6}
            & (1e-3, 10, 500)        & 0.37 & 8.4e-3 & 176.81 & 163.32 \\ \cline{2-6}
            & (1e-3, 10, 1000)       & 0.42 & 8.2e-3   & 328.05 & 327.64  \\ \cline{2-6}
            & (1e-1, 10, 2000)       & 0.52 & 8.7e-3  & 910.73 & 841.25  \\ \hline
            \ChangeRT{0.5 pt}
            \multirow{4}{*}{\makecell{ECOC\\Attack}}& (1e-3,10,100)        & 0.18 & 7.7e-3      & 92.28  & 131.46 \\ \cline{2-6}
		& (1e-3, 10, 500)        & 0.59 & 7.0e-3  & 488.73 & 697.33 \\ \cline{2-6}
		& (1e-3, 10, 1000)       & 0.69 & 6.5e-3  & 967.56 & 1387.41 \\ \cline{2-6}
		& (1e-1, 10, 2000)       & 0.73 & 6.0e-3  & 1921.29 & 2754.74 \\ \hline
\ChangeRT{0.5pt}
          ML-CW & (0.01, 10, 1000)   & 0.64 & 1.4e-2  & 297.53 & 1679.72 \\ \hline
\ChangeRT{0.5pt}
            LOTS & $2000$   & 0.78 & 4.3e-3  & 242.23  &  600.14 \\ \hline
\ChangeRT{0.5pt}
            \multirow{2}{*}{MLA-LP} & $1000$   & 0.00 & NA & NA &  
            \CH{NA} \\ \cline{2-6}
                         & $3000$   & 0.00 & NA & NA &  ($>$1h)  
                         \\ \hline
        \end{tabular}
\end{table}



\subsection{Multi-label classification (VOC2012 - InceptionV3)}

Table \ref{TAB:ML-Real}. shows the results in the case of multi-label classification when the target label vector is chosen randomly among those contained in the training dataset.
%
The average number of labels targeted by the attacks is 2.1\footnote{Most of the images in the VOC2012 dataset contains one or very few labeled objects, hence the label vectors have few 1's and choosing the target label from the training dataset results in few bit changes}.
%
%
It can be observed that both JMA and ML-C\&W attack achieve  ASR = 1. 
However, JMA has an advantage over ML-C\&W both in terms of MSE and, most of all, in terms of computational complexity, being approximately 64 times faster than ML-C\&W.
%
%
LOTS achieves much worse performance, with  ASR = 0.66 and a larger MSE. However, the average bASR is $0.979$, that confirms that the method has some effectiveness.
One possible interpretation for the not so good results of LOTS relies on the construction of the target vector performed by the  method.
In fact, compared to single-label 
classification, where images belonging to the same class have similar content (e.g. images showing the same traffic sign), in the multi-label classification case, images sharing the same label vector may be very different from each other.
For instance, an image of a crowded city street and an image with footballers on a soccer field are both instances of the people category, with the `people' label equal to 1.
As LOTS chooses the target output vector by averaging the logits of the target images, the resulting target vector might not be plausible when the logits of the target images are significantly different.
Finally, the performance of MLA-LP \CH{and SemA-ML}  are also poor, with a much lower ASR w.r.t. JMA and  ML-C\&W. Moreover, \CH{in the case of MLA-LP,} the complexity of the attack is very high\footnote{
These results are not in contrast with the results reported in \cite{zhou2020generating} for the same dataset, given that in \cite{zhou2020generating} the method is validated for target attacks that change only one bit of the original label vector, while the average number of bits changed in our experiments is larger than 2.}.

Tables \ref{TAB:ML-5bit} through \ref{TAB:ML-20bit} show the results in a challenging scenario where the target vector labels are not chosen from the training/test set
but are obtained by randomly changing a certain number of bits of the groud truth label vector.
Specifically,
Tables \ref{TAB:ML-5bit}, \ref{TAB:ML-10bit} and \ref{TAB:ML-20bit}
correspond to attacks inducing respectively 5,  10 and 20 bit errors (in the latter case all the bits should be changed).
%
%
%
%
%
The results are reported for JMA, ML-C\&W, MLA-LP \CH{and SemA-ML}. 
As stated previously, LOTS can not be applied in this case, since it requires the existence of a  number of samples corresponding to the target label vector.

Looking at the results in Table \ref{TAB:ML-5bit} (5 flipping bits), JMA can achieve an ASR = 0.95 (bASR = 0.99), which is 5\% higher than ML-CW, with an average attack time approximately 6 times faster, while MLA-LP \CH{and SemA-ML are} totally ineffective (ASR = 0).
%
%
%
As shown in Table \ref{TAB:ML-10bit},
 in the 10 flipping bit case, JMA can still achieve an ASR close to 0.90,
at the price of a longer time necessary to run the attack, with an average number of iterations $\bar{n}_{it}$ that goes above 240.
%
C\&W attack  achieves a  lower ASR 
when $n_{it,max}$ is raised to 5000, with a very large  computational cost. \CH{The MSE reported for  C\&W  is lower, however, this value is not computed on the same set of imagesand, arguably, a stronger distortion is necessary to attack the additional 16\% images attacked by JMA. 
	}
Fig. \ref{fig:mulitlabel} shows some examples of images successfully attacked by JMA in the 10-bit error case, when half of the labels are modified.

As shown in Table \ref{TAB:ML-20bit}, the ASR drops in the 20 bit case (when the attacker tries to modify all the 20 bits of the label vector).
Arguably, this represents a limit and extremely challenging scenario. It can be observed that, by increasing $n_{it,max}$ to 2000, JMA is successful 26\% of the times, at the price of a high computational cost,  while ML-C\&W is never successful.
Moreover, it is interesting to observe that JMA can achieve a bASR of 0.68, meaning that almost 70\% of the labels are modified by the attack, while the  bASR for ML-C\&W is below 0.1.
Not surprisingly, the MSE is large in this case.


%


\begin{figure*}
  \centering
  \begin{tabular}{@{}c@{}}
    \includegraphics[width=0.7\linewidth]{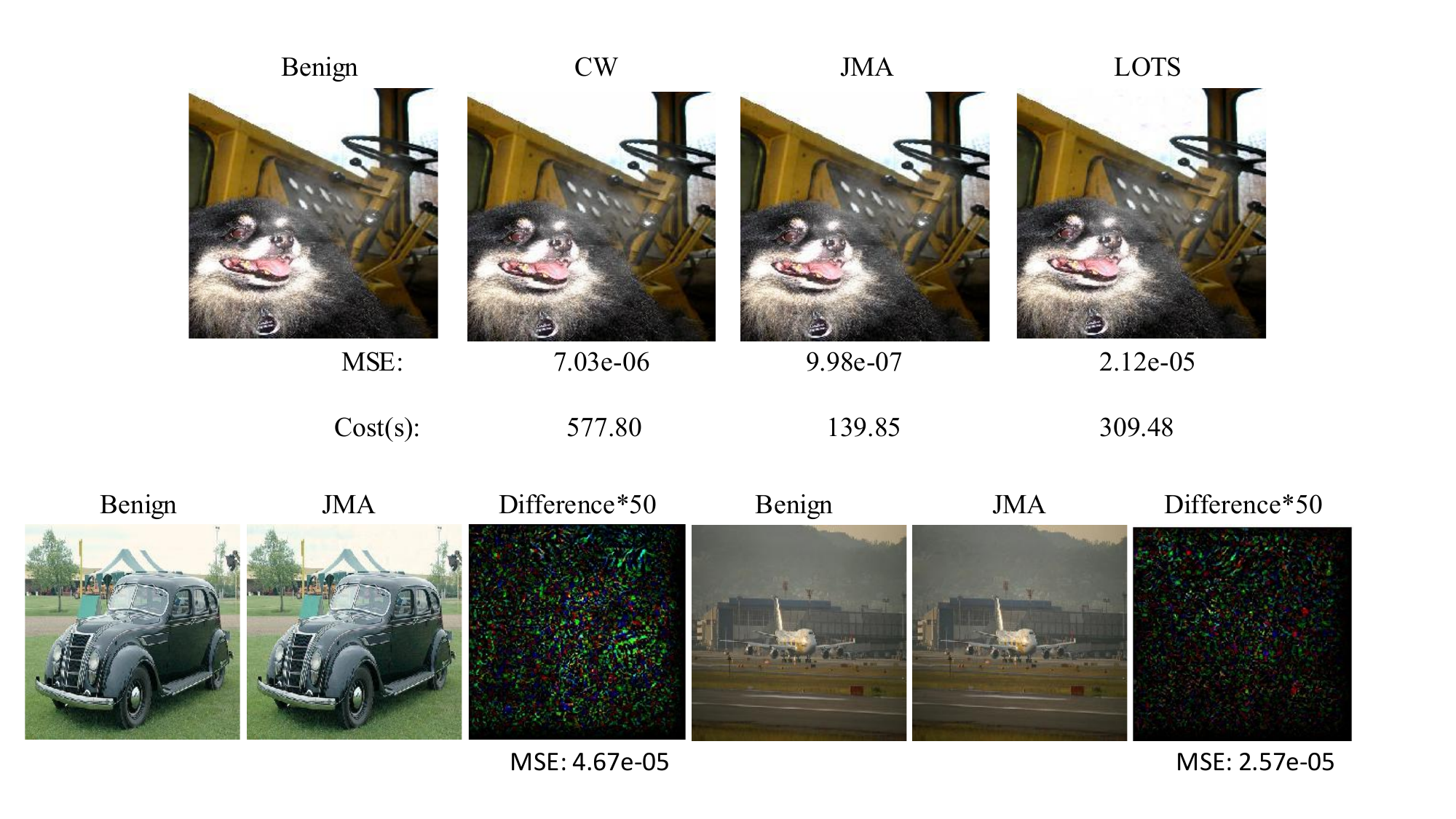}
  \end{tabular}

  \caption{Examples of JMA attacked images for which the attack can successfully flip 10 targeted bits of the label vector in the multi-label (VOC2012) classification task. For the car: the original image has label vector
  \footnotesize \texttt{(car=1, person=0, cat=0, cow=0, dog=0, sheep=0, bike=0, chair=0, potted-plant=0, monitor=0, aeroplane=0, bird=0, horse=0, boat=0, dining-table=0, train=0,...)} \normalsize
  while  the JMA attacked image is classified as \footnotesize \texttt{(car=0, person=1, cat=1, cow=1, dog=1, sheep=1, bike=1, chair=1, potted-plant=1, monitor=1, aeroplane=0, bird=0, horse=0, boat=0, dining-table=0, train=0, ...)}\normalsize. For the airplane: the original image has label vector
    \footnotesize \texttt{(car=0, person=0, cat=0, cow=0,
    dog=0, sheep=0, bike=0, chair=0, potted-plant=0, monitor=0, aeroplane=1, bird=0, horse=0, boat=0, dining-table=0, train=0,...)}\normalsize,
  while  the label vector of the JMA attacked image is      \footnotesize \texttt{(car=0, person=1, cat=1, cow=1, dog=0, sheep=0, bike=0, chair=0, potted-plant=1, monitor=1, aeroplane=0, bird=0, horse=1, boat=1, dining table=1, train=0,...)}\normalsize.
  }  \label{fig:mulitlabel}
  \vspace{-0.3cm}
\end{figure*}

\begin{table}[bh]
	\centering
	\caption{Results against multi-label VOC2012 classification \CH{(the target label vector is taken from the training set)}. On  average, the attacks are asked to flip 2.1 labels/bits.
	\protect\footnotemark
		}
	\label{TAB:ML-Real}
	\resizebox{\columnwidth}{!}{
      \renewcommand\arraystretch{1.2}
        \begin{tabular}{|l|c|c|c|c|c|c|}
		\hline 
		& Parameters & ASR & bASR &  MSE   & $\bar{n}_{it}$ & Time(s)  \\ \hline 
\ChangeRT{0.5pt}
		\multirow{4}{*}{JMA}
& (0.05, 200) & 1.00 & 1.0000   & 2.2e-6      & 29.94  & 25.97 \\ \cline{2-7}
& (0.1, 200) & 1.00 & 1.0000  & 2.6e-6    & 17.50  & 15.39 \\ \cline{2-7}
& (0.5, 200) & 1.00 & 1.0000  & 1.1e-5       & 7.51  & 7.32 \\   \cline{2-7}
& (1, 200) &  1.00 & 1.0000   & 3.3e-5      & 6.42  & 6.45 \\ \hline 
\ChangeRT{0.5pt}
		ML-C\&W & (0.01, 1000) & 1.00 &1.0000    & 6.3e-6 & 180.93 & 469.90 \\ \hline 
\ChangeRT{0.5pt}
		LOTS & 2000  & 0.66 & 0.9787  & 8.1e-6 & 437.20 & 268.16 \\ \hline
\ChangeRT{0.5pt}
		MLA-LP & 100 & 0.58 & 0.9560 & 7.6e-6 & 50.44 & 1361.29 \\ \hline
		\ChangeRT{0.5pt}	
		\CH{SemA-ML} & \CH{300} &   \CH{0.12} &   \CH{	0.8742} & \CH{8.02-5}&  \CH{30.30	}  &   \CH{6.55} \\ \hline
	\end{tabular}
}
\end{table}
\footnotetext{\CH{In this and the following tables, the parameter $n_{bs}$ is omitted for ML-CW, being always equal to 10.}}

\begin{table}[bh!]
	\centering
	\caption{Results for VOC2012  classification with random choice of the target label vector (5 bits out of 20 are flipped).}
	\label{TAB:ML-5bit}
	\resizebox{\columnwidth}{!}{
\renewcommand\arraystretch{1.2}
\begin{tabular}{|l|c|c|c|c|c|c|}
		\hline 
\ChangeRT{0.5pt}
		& Parameters & ASR & bASR & MSE   & $\bar{n}_{it}$ & Time(s)    \\ \hline 
		\multirow{3}{*}{JMA}
         & (0.05, 200) &  0.49 &0.9032  & 1.2e-5 & 129.85 & 139.31 \\ \cline{2-7}
		 &(0.1, 200)  &  0.82 &0.9687  & 1.8e-5 & 110.21 & 106.98 \\   \cline{2-7}
		 & (0.2, 200) &  0.95 &0.9942   & 3.0e-5 & 87.96  & 79.05 \\ \hline 
\ChangeRT{0.5pt}
		ML-C\&W & (0.01, 1000) & 0.90 & 0.9883 & 2.9e-5 & 113.35 & 420.21 \\ \hline 
\ChangeRT{0.5pt}
		MLA-LP & 100 & 0.00 & 0.4264  & NA & NA & NA \\ \hline
				\CH{SemA-ML} & \CH{300} &   \CH{0.00} &   \CH{	0.7412} &  \CH{ NA }  &   \CH{ NA } & \CH{ NA } \\ \hline
	\end{tabular}}
\end{table}

\begin{table}[bh!]
	\centering
	\caption{Results for VOC2012 classification with random choice of the target label vector (10 bits out of 20 are flipped).}
	\label{TAB:ML-10bit}
	\resizebox{\columnwidth}{!}{
\renewcommand\arraystretch{1.2}
\begin{tabular}{|l|c|c|c|c|c|c|}
		\hline 
		& Parameters & ASR  & bASR &  MSE   & $\bar{n}_{it}$ & Time(s)    \\ \hline 
\ChangeRT{0.5pt}
		\multirow{3}{*}{JMA}
		& (0.05,200) & 0.01 & 0.5714  & 1.5e-5 & 161.00 & 157.41 \\ \cline{2-7}
		 & (0.1, 500) & 0.56 & 0.8892& 4.3e-5  & 312.42  & 397.42 \\ \cline{2-7}
		 & (0.5, 500) &  0.88 & 0.9805 &   2.7e-4 & 245.24 &  278.02 \\ \hline
\ChangeRT{0.5pt}
		\multirow{6}{*}{ML-C\&W}
		 & (0.01,  1000) & 0.295& 0.8697 &3.2e-5&270.17&498.63 \\ \cline{2-7}
		 & (0.01, 2000) & 0.4186 & 0.9156 & 3.6e-5 & 451.78& 939.53\\ \cline{2-7}
		  & (0.1,  2000) & 0.475 & 0.9147 &  3.1e-5& 433.75& 934.58\\ \cline{2-7}
		  & (0.5,  2000) & 0.475 & 0.9147 & 3.3e-5 & 430.96& 908.92\\ \cline{2-7}
		 & (0.1, 3000)  & 0.575 & 0.9331 &  3.3e-5  & 581.53  & 1461.05  \\ \cline{2-7}
		 & (0.1, 5000)  & 0.72 & 0.9557 &  3.4e-5  & 775.97  & 2183.79  \\ \hline 
 \ChangeRT{0.5pt}
 		MLA-LP & 100 & 0.00 & 0.5167  &NA & NA & NA \\ \hline
 				\CH{SemA-ML} & \CH{300} &   \CH{0.00} &   \CH{	0.4994} &  \CH{NA}  &   \CH{NA} & \CH{NA} \\ \hline

	\end{tabular}}
\end{table}

\begin{table}[bh!]
	\centering
	\caption{Results for VOC2012  classification  with random choice of the target label vector. Case of 20 bit errors (all labels are changed by the attack).
}
	\label{TAB:ML-20bit}
	\resizebox{\columnwidth}{!}{
\renewcommand\arraystretch{1.2}
\begin{tabular}{|l|c|c|c|c|c|c|}
		\hline 
		& Parameters & ASR & bASR  & MSE &  $\bar{n}_{it}$ & Time(s)   \\ \hline 
		\multirow{7}{*}{JMA}
& (0.05, 200) &  0 & 0.1109      & NA  & NA & NA\\ \cline{2-7}
&(0.1, 200)  &  0 & 0.1276       & NA &  NA & NA \\   \cline{2-7}
& (0.2, 200) &  0 & 0.1475        & NA & NA & NA \\ \cline{2-7} 
& (0.5, 200) &  0 & 0.1845       &  NA & NA & NA \\  \cline{2-7} 
& (0.5, 500) & 0 & 0.3288   &  NA  & NA & NA \\  \cline{2-7}  
& (0.5, 1000) &  0.03 & 0.5068  & 8.2e-4  & 938.16 & 980.82 \\ \cline{2-7}
& (0.5, 2000) &  0.26 & 0.6795  & 1.1e-3 & 1484.34 & 1677.49 \\ \hline
\multirow{2}{*}{ML-C\&W }
 & (0.01, 10, 1000) & 0 & 0.0659  &  NA  & NA & NA \\ \cline{2-7} 
& (0.5, 10, 2000) & 0 & 0.0882  & NA & NA & NA \\ \hline 
	\end{tabular}}
\end{table}


\begin{table}[bh!]
	\centering
	\caption{\CH{Results for VOC2012 using different architectures 
    (target label vector from the training set):  
    (a) ResNet50, (b) ViT-B/16 and (c) CLIP+MLP. 
    }}
	\label{TAB:ML-others-Real}
	%
			\CH{
    	    \begin{subtable}[h]{1\columnwidth}
	\centering
						\resizebox{\columnwidth}{!}{
		\begin{tabular}{|l|c|c|c|c|c|c|c|c|}
			\hline 
			& Parameters & ASR & bASR &  MSE   & SSIM&  $\bar{n}_{it}$ & Time(s)  \\ \hline 
			\ChangeRT{0.5pt}
			\multirow{3}{*}{JMA}
			& (0.05, 200) & 1.00 & 1.0000 & 3.3e-6 & 0.999 & 32.50 & 2.88 \\\cline{2-8}
			& (0.1, 200) & 1.00 & 1.0000 & 3.8e-6 & 0.999 & 21.96 & 1.60 \\\cline{2-8}
			& (0.5, 200) & 1.00 & 1.0000 & 1.4e-5 & 0.993 & 12.84 & 0.79 \\  \cline{2-8}
			& (1, 200) &   1.00 & 1.0000 & 4.0e-5 & 0.982 & 11.84 &  0.80 \\ \hline 
			\ChangeRT{0.5pt}
			ML-C\&W & (0.01,1000) & 0.58 & 0.9583 & 3.1e-5 & 0.989 &  761.00 &  307.27\\ \hline 
			\ChangeRT{0.5pt}
			LOTS & 2000  & 0.83 & 0.9610 & 2.0e-2 & 0.386 & 34.84 & 138.20 \\ \hline
			\ChangeRT{0.5pt}
			MLA-LP & 100 & 0.22 & 0.9500 & 3.5e-4 & 0.964 & 10.98 & 648.39 \\ \hline
			\ChangeRT{0.5pt}	
			\CH{SemA-ML} & \CH{300} &  0.31 & 0.9167 & 9.3e-5 & 0.969 & 5.30 & 5.45 \\ \hline
		\end{tabular}
	}
			\caption{\CH{ResNet50}}
	\label{tab:ResNet}
		\end{subtable}
    \begin{subtable}[h]{1\columnwidth}
	\centering
	\resizebox{\columnwidth}{!}{
				\begin{tabular}{|l|c|c|c|c|c|c|c|c|}
					\hline 
					& Parameters & ASR & bASR &  MSE   & SSIM&  $\bar{n}_{it}$ & Time(s)  \\ \hline 
					\ChangeRT{0.5pt}
					\multirow{3}{*}{JMA}
					& (0.05, 200) & 0.99 & 0.9993 & 4.4e-5 & 0.990 & 26.01 & 3.89 \\ \cline{2-8}
					& (0.1, 200) &  1.00 & 1.0000 & 4.4e-5 & 0.990 & 17.65 & 2.22 \\ \cline{2-8}
					& (0.5, 200) & 1.00 & 1.0000 & 3.9e-5 & 0.991 & 9.49 & 0.85 \\ \cline{2-8}
					& (1, 200) & 1.00 & 1.0000 & 4.6e-5 & 0.989 & 8.48 & 0.69 \\ \hline 
					\ChangeRT{0.5pt}
					ML-C\&W & (0.01, 1000) & 0.81 & 0.9825 & 3.3e-5 & 0.992 & 635.34 & 996.74 \\ \hline 
					\ChangeRT{0.5pt}
					LOTS & 2000  & 0.84 & 0.9750 & 1.5e-2 & 0.603 & 38.00 & 217.00 \\ \hline
					\ChangeRT{0.5pt}
					MLA-LP & 100 & 0.52 & 0.9705 & 6.0e-4 & 0.910 & 5.79 & 787.87 \\ \hline
					\ChangeRT{0.5pt}	
					\CH{SemA-ML} & \CH{300} & 0.16 & 0.8835 & 8.6e-5 & 0.974 & 17.91 & 11.52\\ \hline
				\end{tabular}
			}
						\caption{\CH{ViT-B-16}}
			\label{tab:ViT}
		\end{subtable}
	  \begin{subtable}[h]{1\columnwidth}
		\centering
		\resizebox{\columnwidth}{!}{
			\begin{tabular}{|l|c|c|c|c|c|c|c|c|}
				\hline 
				& Parameters & ASR & bASR &  MSE   & SSIM&  $\bar{n}_{it}$ & Time(s)  \\ \hline 
				\ChangeRT{0.5pt}
				\multirow{3}{*}{JMA}
				& (0.05, 200) & 1.00 & 1.0000 & 1.6e-6 & 0.999 & 28.65 & 4.46 \\ \cline{2-8}
				& (0.1, 200) & 1.00 & 1.0000 & 2.3e-6 & 0.999 & 19.72 & 2.81 \\ \cline{2-8}
				& (0.5, 200) & 1.00 & 1.0000 & 1.2e-5 & 0.994 & 11.88 & 1.30 \\   \cline{2-8}
				& (1, 200) & 1.00 & 1.0000 & 3.5e-5 & 0.984 & 11.07 & 1.13 \\ \hline 
				\ChangeRT{0.5pt}
				ML-C\&W & (0.01, 1000) & 0.14 & 0.9095 & 4.8E-05 & 0.983 & 684.75 & 1613.82 \\ \hline 
				\ChangeRT{0.5pt}
				LOTS & 2000  & 0.97 & 0.9818 & 1.3E-02 & 0.496 & 19.59 & 106.17 \\ \hline
				\ChangeRT{0.5pt}
				MLA-LP & 100 & 0.11 & 0.9400 & 2.1E-04 & 0.977 & 5.57 & 710.85 \\ \hline
				\ChangeRT{0.5pt}	
				\CH{SemA-ML} & \CH{300} &  0.09 & 0.8788 & 9.9E-05 & 0.973 & 1.00 & 21.84\\ \hline
			\end{tabular}
		}
		\caption{\CH{CLIP+MLP}}
		\label{tab:CLIP}
	\end{subtable}
}
\end{table}

\begin{table}[bh!]
	\centering
	\caption{\CH{Results for VOC2012, for a 10-bit flipping attack (10 bits out of 20 are flipped). 
	(a) ResNet50, (b) ViT-B/16 and (c) CLIP+MLP.
	}}
	\label{TAB:ML-others- 10bit}
		\CH{
	 \begin{subtable}[h]{1\columnwidth}
	\resizebox{\columnwidth}{!}{
\begin{tabular}{|l|c|c|c|c|c|c|c|}
	\hline 
	& Parameters & ASR  & bASR &  MSE   & SSIM & $\bar{n}_{it}$ & Time(s)    \\ \hline 
	\ChangeRT{0.5pt}
	\multirow{3}{*}{JMA}
	& (0.05, 200) & 0.00 & 0.5483 & NA & NA & NA	 & NA \\ \cline{2-8} 
	& (0.1, 500) & 0.94 & 0.9780 & 4.8e-5 & 0.984 & 291.08 & 25.76 \\ \cline{2-8}
	& (0.5, 500) & 1.00 & 1.0000 & 2.4e-4 & 0.916 & 166.60 & 14.66 \\ \hline 
	\ChangeRT{0.5pt}
	\multirow{4}{*}{ML-C\&W}
	& (0.01, 1000) & 0.37 & 0.8805 & 5.0e-4 & 0.884 & 631.78 & 275.72 \\ \cline{2-8}
	& (0.5, 2000) & 0.59 & 0.8808 & 1.1e-1 & 0.155 & 1108.74 & 1154.25 \\ \cline{2-8}
	& (0.1, 3000)  & 0.42 & 0.8710 & 1.6e-2 & 0.464 & 581.57 & 1747.41 \\ \cline{2-8}
	& (0.1, 5000)  & 0.45 & 0.8690 & 1.6e-2 & 0.470 & 756.03 & 3340.60 \\ \hline 
	\ChangeRT{0.5pt}
	MLA-LP & 100 & 0.00 & 0.5625 & NA & NA & NA & NA \\ \hline
	\CH{SemA-ML} & \CH{300} & 0.00 & 0.4980 & NA & NA & NA & NA \\ \hline
	\end{tabular}
}
		\caption{\CH{ResNet50}}
\label{tab:ResNet-10}
\end{subtable}
	 \begin{subtable}[h]{1\columnwidth}
	\resizebox{\columnwidth}{!}{
\begin{tabular}{|l|c|c|c|c|c|c|c|}
	\hline 
	& Parameters & ASR  & bASR &  MSE   & SSIM & $\bar{n}_{it}$ & Time(s)    \\ \hline 
	\ChangeRT{0.5pt}
	\multirow{3}{*}{JMA}
	& (0.05, 200) & 0.65 &  0.9073 &  1.1e-4 &  0.974 &  154.64 &  25.00 \\ \cline{2-8}
	& (0.1, 500) & 1.00 &  1.0000 &  1.6e-4 &  0.967 &  126.64 &  19.87 \\ \cline{2-8}
	& (0.5, 500) & 1.00 &  1.0000 &  2.6e-4 &  0.948 &  63.89 &  9.79 \\ \hline 
	\ChangeRT{0.5pt}
	\multirow{4}{*}{ML-C\&W}
	& (0.01, 1000) & 0.51 &  0.8680 &  5.5e-4 &  0.903 &  582.63 &  996.45 \\ \cline{2-8}
	& (0.5, 2000) & 0.74 &  0.8888 &  5.6e-2 &  0.272 &  358.84 &  3230.31 \\ \cline{2-8}
	& (0.1,  3000)  & 0.62 &  0.8445 &  1.2e-2 &  0.584 &  912.28 &  4074.83 \\ \cline{2-8}
	& (0.1, 5000)  &0.66 &  0.8445 &  1.9e-2 & 0.565 & 1497.76 &  6674.43 \\ \hline 
	\ChangeRT{0.5pt}
	MLA-LP & 100 & 0.00 &  0.5635 &  NA &  NA &  NA &  NA \\ \hline
	\CH{SemA-ML} & \CH{300} & 0.00 &  0.4995 &  NA &  NA &  NA &  NA \\ \hline
\end{tabular}
	}
	\caption{\CH{ViT-B/16}}
	\label{tab:VIT-10}
\end{subtable}
	 \begin{subtable}[h]{1\columnwidth}
	\resizebox{\columnwidth}{!}{
\begin{tabular}{|l|c|c|c|c|c|c|c|}
	\hline 
	& Parameters & ASR  & bASR &  MSE   & SSIM & $\bar{n}_{it}$ & Time(s)    \\ \hline 
	\ChangeRT{0.5pt}
	\multirow{3}{*}{JMA}
	& (0.05, 200) & 1.00 & 0.9975 & 3.8e-6 & 0.998 & 59.16 & 8.70 \\ \cline{2-8}
	& (0.1, 500) & 1.00 & 1.0000 & 5.8e-6 & 0.997 & 41.45 & 5.86 \\ \cline{2-8}
	& (0.5, 500) & 1.00 & 1.0000 & 3.3e-5 & 0.986 & 22.73 & 2.79 \\ \hline 
	\ChangeRT{0.5pt}
	\multirow{4}{*}{ML-C\&W}
	& (0.01, 1000) & 0.51 & 0.8680 & 5.5e-4 & 0.903 & 582.63 & 2052.44 \\ \cline{2-8}
	& (0.5, 2000) & 0.74 & 0.8888 & 5.6e-2 & 0.272 & 358.84 & 3600.40 \\ \cline{2-8}
	& (0.1,  3000)  & 0.62 & 0.8445 & 1.2e-2 & 0.584 & 912.28 & 4149.43 \\ \cline{2-8}
	& (0.1, 5000)  & 0.66 & 0.8445 & 1.9e-2 & 0.565 & 1497.76 & 4867.13 \\ \hline 
	\ChangeRT{0.5pt}
	MLA-LP & 100 & 0.00 & 0.5755 & NA & NA & NA & NA \\ \hline
	\CH{SemA-ML} & \CH{300} &  0.00 & 0.4970 & NA & NA & NA & NA \\ \hline
\end{tabular}
	}
	\caption{\CH{CLIP+MLP}}
	\label{tab:CLIP-10}
\end{subtable}
}
\end{table}

\subsection{\CH{Multi-label classification (VOC2012 - ResNet50, ViT-B/16, CLIP+MLP)}}
\label{sec.furtherExp}



\CH{To validate the effectiveness of JMA against different types of architectures, we also trained 
a residual network (ResNet50), a Transformer-based architecture (ViT-B/16) and a network based on pre-trained vision-language models (CLIP), obtained by adding a lightweight classification head - namely a Multi-Layer Perceptron - on top of CLIP features (CLIP+MLP).}
\CH{The models were trained using the multi-label cross-entropy loss with Adam optimizer with lr = $10^{-3}$, batch size of 32 samples, 
and early stopping based on the mAP computed on the validation set.
The ResNet50, ViT-B/16, and CLIP+MLP models obtained a mAP score on the test set of 0.94, 0.87, and 0.90, respectively.}

\CH{Table \ref{TAB:ML-others-Real} and  \ref{TAB:ML-others- 10bit} report the results respectively in the case of attack target vector chosen randomly from the training set and in the case where 10 bit at random are modified in the label vector. 
For a more comprehensive evaluation on the quality of  attacked images and the amount of distortion introduced by the attack, in the tables we also report the SSIM in addition to MSE.
We see that the state-of-the-art methods, especially ML-C\&W and  MLA-LP, achieve worse performance than in the InceptionV3 case, while LOTS is more effective. With regard to JMA, performance are very good and  ASR = 1 can be obtained in all the cases and in both attack settings.  Furthermore, JMA is the attack introducing the lowest distortion, with an MSE 
in the range $[10^{-6}, 10^{-4}]$. 
The SSIM is also very good, being always above 0.99 in the first attack setting, and remaining good also in the challenging 10-bit flips attack. Compared to the InceptionV3 case,  JMA takes more iterations to run ($\bar{n}_{it}$ is larger), 
however the algorithm is extremely fast and the attacked image can be found in a few seconds in most cases.}

\begin{table}[bh!]
	\centering
	\caption{\CH{Results of the 10-bit flipping attack on MS-COCO dataset.
		}}
	\label{TAB:ML-10bit-COCO}
	\CH{
		\begin{subtable}[h]{1\columnwidth}
			\resizebox{\columnwidth}{!}{
				\begin{tabular}{|l|c|c|c|c|c|c|c|}
					\hline 
					& Parameters & ASR  & bASR &  MSE   & SSIM & $\bar{n}_{it}$ & Time(s)    \\ \hline 
					\ChangeRT{0.5pt}
					\multirow{3}{*}{JMA}
					& (0.05,200) & 1.00 & 1.0000 & 7.7e-6 & 0.997 & 53.80 & 76.86 \\ \cline{2-8}
					& (0.1, 500) & 1.00 & 1.0000 & 9.5e-6 & 0.997 & 38.59 & 51.66 \\ \cline{2-8}
					& (0.5, 500) & 1.00 & 1.0000 & 4.4e-5 & 0.984 & 30.25 & 40.24   \\ \hline 
					\ChangeRT{0.5pt}
					\multirow{3}{*}{ML-C\&W}
					& (0.01, 1000) & 0.00 & 0.7649 & NA & NA & NA & NA \\ \cline{2-8}
					& (0.5, 2000) & 0.00 & 0.7984 & NA & NA & NA & NA\\ \cline{2-8}
					& (0.1,  3000)  & 0.00 & 0.7578 & NA & NA & NA & NA \\ \cline{2-8}
					\ChangeRT{0.5pt}
					MLA-LP & 100 & 0.00 & 0.8892 & NA & NA & NA & NA  \\ \hline
					\CH{SemA-ML} & 300 & 0.00 & 0.8518 & NA & NA & NA & NA \\ \hline
				\end{tabular}
			}
			\caption{\CH{Swin-B}}
			\label{tab:COCO-swin}
		\end{subtable}
		\begin{subtable}[h]{1\columnwidth}
			\resizebox{\columnwidth}{!}{
				\begin{tabular}{|l|c|c|c|c|c|c|c|}
					\hline 
					& Parameters & ASR  & bASR &  MSE   & SSIM & $\bar{n}_{it}$ & Time(s)    \\ \hline 
					\ChangeRT{0.5pt}
					\multirow{3}{*}{JMA}
					& (0.05, 200) & 0.46 & 0.9177 & 8.9e-6 & 0.996 & 149.17 & 98.73 \\ \cline{2-8}
					& (0.1, 500) & 1.00 & 1.000 & 1.9e-5 & 0.993 & 163.08 & 109.78 \\ \cline{2-8}
					& (0.5, 500) & 1.00 & 1.000 & 1.4e-4 & 0.957 & 101.66 & 66.38\\ \hline 
					\ChangeRT{0.5pt}
					\multirow{3}{*}{ML-C\&W}
					& (0.01, 1000) & 0.21 & 0.9373 & 5.3e-4 & 0.905 & 320.93 & 1377.91 \\ \cline{2-8}
					& (0.5, 2000) & 0.11 & 0.9476 & 1.1e-1 & 0.222 & 853.68 & 3154.67 \\ \cline{2-8}
					& (0.1,  3000)  & 0.19 & 0.9331 & 2.1e-2 & 0.481 & 793.53 & 4507.12 \\ \cline{2-8}
					\ChangeRT{0.5pt}
					MLA-LP & 100 & 0.00 & 0.8816 & NA & NA & NA & NA  \\ \hline
					\CH{SemA-ML} & \CH{300} & 0.00 & 0.8596 & NA & NA & NA & NA \\ \hline
				\end{tabular}
			}
			\caption{\CH{CLIP+MLP}}
			\label{tab:COCO-clip}
		\end{subtable}
	}
\end{table}

\begin{table}[bh!]
	\centering
	\caption{\CH{Results of the 10-bit flipping attack on NUS-WIDE dataset (the network is CLIP+MLP).}}
	\label{TAB:ML-10bit-NUS-WIDE}
	\CH{
		\resizebox{\columnwidth}{!}{
			\begin{tabular}{|l|c|c|c|c|c|c|c|}
				\hline 
				& Parameters & ASR  & bASR &  MSE   & SSIM & $\bar{n}_{it}$ & Time(s)    \\ \hline 
				\ChangeRT{0.5pt}
				\multirow{3}{*}{JMA}
				& (0.05,200) & 0.99 & 0.9975 & 5.1e-6 & 0.998 & 81.65 & 48.80 \\ \cline{2-8}
				& (0.1, 500) & 1.00 & 1.00 & 7.8e-6 & 0.996 & 59.96 & 35.23 \\ \cline{2-8}
				& (0.5, 500) & 1.00 & 1.00 & 5.3e-5  & 0.976 & 36.53 & 20.35 \\ \hline 
				\ChangeRT{0.5pt}
				\multirow{3}{*}{ML-C\&W}
				& (0.01, 1000) & 0.00 & 0.7559 & NA & NA & NA & NA\\ \cline{2-8}
				& (0.5, 2000) & 0.00 & 0.7612& NA & NA & NA & NA \\ \cline{2-8}
				& (0.1,  3000)  & 0.00 & 0.7070& NA & NA & NA & NA \\ \cline{2-8}
				\ChangeRT{0.5pt}
				MLA-LP & 100 & 0.00 & 0.8843  & NA &  NA & NA & NA \\ \hline
				\CH{SemA-ML} & \CH{300} & 0.00 &  0.8630 &  NA & NA  &   NA & NA \\ \hline
			\end{tabular}
		}
	}
\end{table}

\begin{table}[bh!]
	\centering
	\caption{\BT{Comparison between JMA and SemA-ML \cite{Mahmood_2024_CVPR} on the NUS-WIDE dataset in the case of semantic consistent multi-label attacks. 
	}}
	\label{TAB:comparisonCVPR}
	\CH{
		\begin{subtable}[h]{1\columnwidth}
			\resizebox{\columnwidth}{!}{
				\begin{tabular}{|l|c|c|c|c|c|c|c|}
					\hline 
					\ChangeRT{0.5pt}
					& Parameters & ASR & bASR & MSE &SSIM &   $\bar{n}_{it}$ & Time(s)    \\ \hline 
                    \multirow{2}{*}{JMA}
                    & (0.05, 200) & 1.00 & 1.0000 & 1.2e-6 & 0.999 & 26.98 & 14.16 \\ \cline{2-8}
                    & (0.1, 500) & 1.00 & 1.0000 & 1.8e-6 & 0.999 & 19.57 & 9.39 \\ \cline{2-8}
                    & (0.5, 500) & 1.00 & 1.0000 & 1.2e-5 & 0.994 & 12.33 & 4.76 \\ \hline
                    \CH{SemA-ML} & \CH{300} & \CH{0.20} & \CH{0.9568} & \CH{7.2e-5} & \CH{0.978} & \CH{36.43} & \CH{5.10} \\ \hline
				\end{tabular}
			}
			\caption{\CH{Target label vector from training dataset. The average number of bits that the attack is required to flip is 4.37 bit.}}
			\label{tab.comparison-a}
		\end{subtable}
		\begin{subtable}[h]{1\columnwidth}
			\resizebox{\columnwidth}{!}{
				\begin{tabular}{|l|c|c|c|c|c|c|c|}
					\hline 
					\ChangeRT{0.5pt}
                    & Parameters & ASR & bASR & MSE &SSIM &   $\bar{n}_{it}$ & Time(s)    \\ \hline 
                    \multirow{2}{*}{JMA}
                    & (0.05, 200) & 1.00 & 1.0000 & 6.6e-7 & 1.000 & 19.03 & 8.84 \\ \cline{2-8}
                    & (0.1, 500) & 1.00 & 1.0000 & 1.0e-6 & 0.999 & 14.17 & 5.77 \\ \cline{2-8}
                    & (0.5, 500) & 1.00 & 1.0000 & 6.6e-6 & 0.996 & 9.58 & 2.88 \\ \hline
                    \CH{SemA-ML} & \CH{300} & \CH{0.48} & \CH{0.9709} & \CH{5.8e-5} & \CH{0.975} & \CH{13.65} & \CH{4.20} \\ \hline
				\end{tabular}
			}
			\caption{\CH{Attack removing all labels/bits. The average number of bits that  the attack is required to flip is 2.90 bit.} 
            }
			\label{tab.comparison-b}
		\end{subtable}
	}
\end{table}

\subsection{\CH{Multi-label classification (MS-COCO, NUS-WIDE)}}
\label{sec.furtherExpbis}

\CH{To validate the effectiveness of JMA against more challenging datasets, we considered the MS-COCO  \cite{lin2014microsoft}  and NUS-WIDE  \cite{NUS-WIDE}  
datasets, containing 80 and 81 labels respectively, resulting in a huge number of possible label vectors. For the MS-COCO dataset, we  trained a  Swin-B  architecture (which obtained a better mAP than the ViT-B/16) and CLIP+MLP, with input size of $224 \times 224$ pixels. The models were trained on the original split of the dataset using the Asymmetric Loss (ASL) proposed in \cite{benbaruch2020asymmetric}, with Adam optimizer, lr = $10^{-3}$ and batch size of 32 samples. 
Early stopping on the validation set was employed based on mAP. 
The Swin-B and CLIP+MLP models scored an mAP of about 0.82 and 0.81, which is aligned with benchmarks for similar input shapes}.
\CH{For the NUS-WIDE case, we split the dataset into 129,431 training, 32,358 validation, and 107,759 test images.
A CLIP+MLP architecture was trained, using the same learning setting as for MS-COCO,
obtaining a mAP score of 0.59, which is aligned with banchmarks.
}


\CH{Table \ref{TAB:ML-10bit-COCO} reports the results on MS-COCO in the  10-bit flipping attack scenario. We see that the architecture has a noticeable impact on the  performance of the attacks. For the state-of-the-art methods, attacking the Swin-B network is more difficult. 
Specifically, the bASR values  achieved by ML-C\&W  range from 0.75 to 0.80 and from 0.93 to  0.95 across the various parameter settings respectively in the Swin-B and CLIP+MLP case, thus being similar to the VOC2012 case. However, the ASR is much lower and only very few images can be fully attacked in the CLIP+MLP case, while in the case of Swin-B the  ASR is 0.
On the contrary, MLA-LP and SemA-ML work better on MS-COCO  w.r.t. the VOC2012 case and  the bASR is higher. 
Regarding JMA, in both cases the attack can reach ASR = 1 with low attack distortion,  thus surpassing the state-of-the-art. In the CLIP+MLP case, the attack takes more time to run and the distortion introduced by the attack is larger, yet the SSIM remains above 0.99.}

\CH{For the NUS-WIDE dataset, the results of the 10-bit flipping attack are reported in Table \ref{TAB:ML-10bit-NUS-WIDE}. We see that the ASR is always 0 for the state-of-the-art methods, while attacking with JMA resulted in ASR = 1. Results in terms of distortion and computational efficiency  are similar to the best results on MS-COCO, achieved using Swin-B.
}

\CH{It is worth observing that SemA-ML is a multi-label attack which enforces that semantic consistency among the labels 
is retained by the attack.
Differently from the case of VOC2012 and MS-COCO, in the NUS-WIDE dataset,  labels 
correspond to both categories and subcategories (e.g. , `vehicle'  and `car'  are distinct labels).
Hence, semantic consistence may not be verified when bits are flipped at random (e.g.,  the label `car' may be turned on while the label `vehicle' remains off, leading to a semantically inconsistent vector).
Hence, to compare JMA with SemA-ML in a setting that satisfies this method's working assumptions, we also ran some tests considering the following simple attack settings, 
where  semantic consistency is satisfied:
i) target label vector chosen at random from the training/test set; ii) target vector with all labels/bits set to zero. The results in these cases are reported in Table \ref{TAB:comparisonCVPR}.   SemA-ML can now attack most of the labels and reach the target in some cases (and the ASR is no longer 0). This is especially the case 
when the attack is asked to set all the bits to zero (in this case, the attack 
	involves flipping a lower number of bits - 2.90 on the average).
    In this case, SemA-ML can achieve bASR = 0.971 and ASR = 0.48. In any case, the performance of JMA  remains largely superior, achieving ASR = 1 in all the settings with a lower MSE and an SSIM above 0.99 (still with a lower complexity).
}


\subsection{\CH{JMA as a one-step attack}}
\label{sec.dis}

Experiments show that JMA requires a small number of iterations to find an adversarial example, much smaller than the other algorithms, and gets an ASR that is generally higher (in some cases, much higher) for a similar distortion.
It is worth stressing that the number of updates $\bar{n}_{it}$ reported in the tables takes into account also the steps of the binary search carried out at the end of the iterations (see Section \ref{sec.attackParam}).
Specifically,   $it = (v-1) + n_{bs}$, where $v$ is the number of `for' loops, that is, the number of image updates (see Algorithm 1), and $n_{bs}$  is set equal  to $6$.  Therefore,
JMA can indeed find an adversarial image in {\em one shot} in 
a considerable number of cases, see for instance Table \ref{C4-TAB:ECOC-CIFAR} for CIFAR-10 ECOC classification and Table \ref{TAB:ML-Real} for the multi-label classification, where  $\bar{n}_{it}$ is lower than 7 for the  attack settings with larger $\epsilon$\footnote{Note that, even when a one-shot attack is possible, a perturbation applied to the image  with a  weak strength $\epsilon$ may not result in an adversarial image, thus requiring that more iterations  are run.}.
This confirms a-posteriori the validity of one of the most basic assumptions underlying our approach, namely the local linear behavior, that, when it holds, allows to perform the attack in one-shot.
When the linear approximation does not hold in the close vicinity of the initial point, and possibly of other  points during the updates, the perturbation is added with a small $\epsilon$ in order to change the initial point, and the Jacobian is recomputed (see Algorithm 1). By inspecting the tables (see for instance Table \ref{TAB:one-hot}), we can see that, when this happens, it is often preferable to use a not-too-small $\epsilon$ so that the point is moved farther from the initial point, where the behavior of the local function can be closer to linear, instead of remaining in the vicinity of the initial point where the local approximation may still be not satisfied.
Doing so, a high ASR can be achieved  with a reduced number of iterations and a similar distortion.

\section{\CH{Concluding remarks}}
\label{sec.con}

We have proposed a  general, nearly optimal, targeted attack, that can solve the original formulation of the adversarial attack by Szegedy et al. under a first order approximation of the network function.
The method resorts to
the minimization of a Mahalanobis distance term, which
depends on the Jacobian matrix taking into account the effort
necessary to move the feature representation of the image in a given direction.
By exploiting the Wolfe
duality, the minimization problem is reduced to a
non-negative least square (NNLS) problem, that can be solved
numerically.
The experiments show that the JMA attack is effective against a wide variety of DNNs adopting different output encoding schemes, including
networks using error correction output encoding (ECOC) and in particular multi-label classification,  outperforming existing attacks in terms of higher ASR, lower distortion and lower complexity,
with attack capabilities far exceeding those of existing attacks.
JMA remains effective also  in the case of one-hot encoding, with much reduced computational complexity with respect for instance  to the  C\&W attack.

\CH{
Finally, we observe that JMA can inspire new defences.
In particular, given its efficiency, JMA can be used to reduce the load of adversarial training and develop effective adversarial training defences \cite{goodfellow2014explaining, wu2017adversarial,kuang2024defense}.
Furthermore, while with the exception of scattered works (\cite{wu2017adversarial}) adversarial training is primarly limited to applications of single-label classification, 
JMA can be used to implement effective adversarial training  also for other types of classifiers (e.g. multi-label classifiers).
In addition, thanks to its efficiency,  JMA could also be used to perform adversarial training of provably robust multi-label classifiers exploiting randomized smoothing \cite{jia2022multiguard}. 
In fact, similarly to what has been done for standard single-label randomized smoothing classifiers (\cite{salman2019provably}),
by training with JMA attacked samples it could be possible to boost the provable robustness of smoothed multi-label classifiers.
}
Future work could also focus on the extension of JMA to a black-box attack scenario, to develop powerful targeted attacks with certain transferability properties.

\bibliographystyle{IEEEtran}
\bibliography{bare_jrnl.bib}

\end{document}